\documentclass{article}

\usepackage{times}
\usepackage{graphicx} 
\usepackage{subfigure} 

\usepackage{natbib}

\usepackage{algorithmic}
\usepackage{algorithm}

\usepackage{hyperref}

\usepackage[nofootnote]{icml2016}

\usepackage{amsmath,amsfonts,graphpap,amscd,mathrsfs,lscape}
\usepackage{epsfig,amssymb,amstext,xspace}
\usepackage{float}	
\usepackage{hyperref}

\usepackage{color}              

\usepackage{color-edits}
\addauthor{cd}{green}
\addauthor{vs}{red}
\addauthor{ak}{blue}

\usepackage{amsthm}

\newcommand{\E}{\mathbb{E}}

\newtheorem{theorem}{Theorem}
\newtheorem{lemma}[theorem]{Lemma}

\newtheorem{corollary}[theorem]{Corollary}

\theoremstyle{definition}
\newtheorem{defn}{Definition}

\newtheorem{example}{Example}

\def\squareforqed{\hbox{\rule{2.5mm}{2.5mm}}}

\def\QED{\ifmmode\squareforqed 
  \else{\nobreak\hfil   
    \penalty50                 
    \hskip1em                  
    \null                      
    \nobreak                   
    \hfil                      
    \squareforqed              
    \parfillskip=0pt           
    \finalhyphendemerits=0     
    \endgraf}                  
  \fi}

\def\blksquare{\rule{2mm}{2mm}}
\def\qedsymbol{\blksquare}
\newcommand{\bg}[1]{\medskip\noindent{\bf #1}}
\newcommand{\ed}{{\hfill\qedsymbol}\medskip}

\newcommand{\R}{\ensuremath{\mathbb R}}

\newcommand{\A}{\ensuremath{\mathcal{A}}}

\newcommand{\F}{\ensuremath{\mathcal H}}

\newcommand{\argmin}{\ensuremath{\mathrm{argmin}}}
\newcommand{\arginf}{\ensuremath{\mathrm{arginf}}}
\newcommand{\Ell}{\ensuremath{\mathcal{L}}}

\newcommand{\PP}{\ensuremath{\text{{\bf P}}}}

 {}

\newcommand{\junk}[1]{}

\newlength{\tmp} \newlength{\lpsx} \newlength{\lpsy} \newlength{\upsx} \newlength{\upsy}

\newcommand{\Or}[1]{\ensuremath{M\left(#1\right)}}

\newcommand{\Omit}[1]{}

\newcommand{\FTPL}{\ensuremath{{\tt FTPL}}}
\newcommand{\BTPL}{\ensuremath{{\tt BTPL}}}

\newcommand{\contalg}{\ensuremath{\textsc{Context-FTPL}}}

\newcommand{\contsemband}{\ensuremath{\textsc{Context-Semi-Bandit-FTPL}}}

\newcommand{\regret}{\ensuremath{\textsc{Regret}}}
\newcommand{\error}{\ensuremath{\textsc{Error}}}
\newcommand{\stb}{\ensuremath{\textsc{Stability}}}

\newcommand{\dotp}[2]{\ensuremath{\langle #1,#2\rangle}}

\newcommand{\contlin}{\ensuremath{\textsc{Context-FTPL}}}
\newcommand{\ftpl}{\ensuremath{\textsc{FTPL}}}

\newcommand{\pred}{Q}

\newcommand{\sectref}[1]{Section~\ref{#1}}

\newcommand{\param}{\ensuremath{y}}
\newcommand{\Param}{\ensuremath{\cal{Y}}}

\icmltitlerunning{Efficient Algorithms for Adversarial Contextual Learning}

\begin{document} 

\twocolumn[
\icmltitle{Efficient Algorithms for Adversarial Contextual Learning}

\icmlauthor{Vasilis Syrgkanis}{vasy@microsoft.com}
\icmladdress{Microsoft Research,
            641 Avenue of the Americas, New York, NY 10011 USA}
\icmlauthor{Akshay Krishnamurthy}{akshaykr@cs.cmu.edu}
\icmladdress{Microsoft Research,
            641 Avenue of the Americas, New York, NY 10011 USA}
\icmlauthor{Robert E. Schapire}{schapire@microsoft.com}
\icmladdress{Microsoft Research,
            641 Avenue of the Americas, New York, NY 10011 USA}


\vskip 0.3in
]

\begin{abstract}

We provide the first oracle efficient sublinear regret algorithms for adversarial versions of the contextual bandit problem. In this problem, the learner repeatedly makes an action on the basis of a context and receives reward for the chosen action, with the goal of achieving reward competitive with a large class of policies. We analyze two settings: i) in the transductive setting the learner knows the set of contexts a priori, ii) in the small separator setting, there exists a small set of contexts such that any two policies behave differently in one of the contexts in the set. Our algorithms fall into the follow the perturbed leader family \cite{Kalai2005} and achieve regret $O(T^{3/4}\sqrt{K\log(N)})$ in the transductive setting and $O(T^{2/3} d^{3/4} K\sqrt{\log(N)})$ in the separator setting, where $K$ is the number of actions, $N$ is the number of baseline policies, and $d$ is the size of the separator. We actually solve the more general adversarial contextual semi-bandit linear optimization problem, whilst in the full information setting we address the even more general contextual combinatorial optimization. We provide several extensions and implications of our algorithms, such as switching regret and efficient learning with predictable sequences.


\end{abstract}

\section{Introduction}\label{sec:intro}
We study contextual online learning, a powerful framework that
encompasses a wide range of sequential decision making problems.
Here, on every round, the learner receives contextual information
which can be used as an aid in selecting an action.
In the full-information version of the problem, the learner then
observes the loss that would have been suffered for each
of the possible actions, while in the much more challenging bandit
version, only the loss that was actually incurred for the chosen
action is observed.
The contextual bandit problem is of particular practical relevance,
with applications to personalized recommendations, clinical trials,
and targeted advertising.

Algorithms for contextual learning, such as Hedge~\cite{Freund1997,cesa1997use} and Exp4~\cite{Auer1995},
are well-known to have
remarkable theoretical properties, being effective even in
adversarial, non-stochastic environments, and capable of performing
almost as well as the best among an exponentially large family of
{\em policies}, or rules for choosing actions at each step.
However, the space requirements and running time of these algorithms
are generally linear in the number of policies, which is far too
expensive for a great many applications which call for an extremely
large policy space.
In this paper, we address this gap between the statistical
promise and computational challenge of algorithms for
contextual online learning in an adversarial setting.

As an approach to solving online learning problems, we
posit that the corresponding batch version is solvable.
In other words, we assume access to a certain optimization oracle 
for solving an associated batch-learning problem.
Concrete instances of such an oracle include empirical risk
minimization procedures for supervised learning, algorithms for the
shortest paths problem, and dynamic programming.

Such an oracle is central to the Follow-the-Perturbed-Leader
algorithms of~\citet{Kalai2005}, although these
algorithms are not generally efficient since they require separately
``perturbing'' each policy in the entire space.
Oracles of this kind have also been used in designing efficient contextual
bandit algorithms~\cite{Agarwal2014,langford2008epoch,dudik2011efficient};
however, these require a much more benign setting in which contexts and
losses are chosen randomly and independently rather than by an
adversary.

In this paper, for a wide range of problems, we present
computationally efficient algorithms for contextual online
learning in an adversarial setting, assuming oracle access.
We give results for both the full-information and bandit settings.
To the best of our knowledge, these results are the first of their
kind at this level of generality.

\paragraph{Overview of results.}
We begin by proposing and analyzing in \sectref{sec:oracles} a new and general
Follow-the-Perturbed-Leader algorithm in the style of~\citet{Kalai2005}.
This algorithm \emph{only} accesses the policy class using the
optimization oracle.

We then apply these results in \sectref{sec:cont-lin} to two settings.
The first is a \emph{transductive setting}~\cite{ben1997online} in
which the learner knows the set of arriving contexts a priori, or,
less stringently, knows only the set, but not necessarily the actual
sequence or multiplicity with which each context arrives.
In the second, \emph{small-separator} setting, we assume that the
policy space admits the
existence of a small set of contexts, called a \emph{separator}, such
that any two policies differ on at least one context from the set.
The size of the smallest separator for a particular policy class can
be viewed as a new measure of complexity, different from the VC
dimension, and potentially of independent interest.

We study these for a generalized online learning problem
called {\em online combinatorial optimization}, which includes as special
cases transductive contextual experts, online
shortest-path routing, online linear optimization~\cite{Kalai2005},
and online submodular minimization~\cite{Hazan2012}.

In \sectref{sec:contextual-bandits}, we extend our results to the bandit setting, or in
fact, to the more general semi-bandit setting, using a technique of~\citet{neu2013efficient}.
Among our main results, we obtain regret bounds for the adversarial
contextual bandit problem of
$O(T^{3/4}\sqrt{K\log(N)})$ in the transductive setting,
and $O(T^{2/3} d^{3/4} K\sqrt{\log(N)})$ in the small-separator
setting,
where $T$ is the number of time steps, $K$ the number of actions,
$N$ the size of the policy space, and $d$ the size of the separator.
Being sublinear in $T$, these bounds imply the learner's performance
will eventually be almost as good as the best policy, although
they are worse than the generally optimal dependence on $T$ of $O(\sqrt{T})$,
obtained by many of the algorithms mentioned above.
On the other hand, these preceding algorithms
are computationally intractable when the policy space is gigantic, while ours runs in
polynomial time, assuming access to an optimization oracle.
Improving these bounds without sacrificing computational efficiency
remains an open problem.

In \sectref{sec:switching}, we give an efficient algorithm when regret is measured
in comparison to a
competitor that is allowed to switch from one policy to another a bounded
number of times.
Here, we show that the optimization oracle can be efficiently
implemented given an oracle for the original policy class.
Specifically, this leads to a fully efficient algorithm for the online
switching shortest path problem in directed acyclic graphs. 

Finally, \sectref{sec:path-length} shows how ``path length'' regret bounds can be derived in
the style of \citet{Rakhlin2013a}.
Such bounds have various applications, for instance, in obtaining
better bounds for playing repeated games~\cite{Rakhlin2013,syrgkanis2015fast}.

\paragraph{Other related work.}
Contextual, transductive online learning using an optimization oracle was previously
studied by \citet{kakade2005batch}, whose work was later 
extended and improved by \citet{cesa2011efficient} using
a generalization of a technique from \citet{cesa1997use}.
However, these previous results are for binary classification or other convex losses
defined on one-dimensional predictions and outcomes; as such, they are
special cases of the much more general setting we consider in the
present paper.

\citet{awerbuch2008online} present an efficient algorithm for the online
shortest paths problem.
This can be viewed as solving an adversarial bandit problem with a very
particular optimization oracle over an exponentially large but highly
structured space of ``policies'' corresponding to paths in a graph.
However, their setting is clearly far more restrictive and structured
than ours is. 

\section{Online Learning with Oracles}\label{sec:oracles}

We start by analyzing the family of Follow the Perturbed Leader algorithms in a very general online learning setting. Parts of this generic formulation follow the recent formulation of \citet{Daskalakis2015}, but we present a more refined analysis which is essential for our contextual learning result in the next sections. The main theorem of this section is essentially a generalization of Theorem 1.1 of \citet{Kalai2005}.

 Consider an online learning problem where at each time-step an adversary picks an outcome $\param^t\in \Param$ and the algorithm picks a policy $\pi^t\in \Pi$ from some policy space $\Pi$.\footnote{We refer to the choice of the learner as a policy, for uniformity of notation with subsequent sections, where the learner will choose some policy that maps contexts to actions.}
The algorithm receives a loss: $\ell(\pi^t,\param^t)$, which could be positive or negative. At the end of each iteration the algorithm observes the realized outcome $\param^t$. 
We will denote with $\param^{1:t}$ a sequence of outcomes $\{\param^1,\param^2,\ldots,\param^t\}$. Moreover, we denote with: 
\begin{equation}
\Ell(\pi,\param^{1:t}) = \sum_{\tau=1}^t \ell(\pi,\param^\tau),
\end{equation} 
the cumulative utility of a fixed policy $\pi\in \Pi$ for a sequence of choices $\param^{1:t}$ of the adversary. 
The goal of the learning algorithm is to achieve loss that is competitive with the best fixed policy in hindsight.
As the algorithms we consider will be randomized, we will analyze the expected regret,
\begin{eqnarray}
\regret =  \sup_{\pi^\star\in \Pi}\E\left[\sum_{t=1}^T\ell(\pi^t,\param^t) -\sum_{t=1}^T\ell(\pi^\star,\param^t)\right],
\end{eqnarray}
which is the worst case difference between the cumulative loss of the learner and the loss of any fixed policy $\pi \in \Pi$. 

We consider adversaries that are \emph{adaptive}, which means that they can choose the outcome $\param^t$ at time $t$, using knowledge of the entire history of interaction. 
The only knowledge not available to an adaptive adversary is any randomness used by the learning algorithm at time $t$.
In contrast, an \emph{oblivious} adversary is one that picks the sequence of outcomes $\param^{1:T}$ before the start of the learning process. 

To develop computationally efficient algorithms that compete with large sets of policies $\Pi$, we assume that we are given oracle access to the following optimization problem.
\begin{defn}[Optimization oracle] 
Given outcomes $\param^{1:t}$ compute the fixed optimal policy for this sequence: 
\begin{equation}
\Or{\param^{1:t}} = \argmin_{\pi\in \Pi} \Ell(\pi,\param^{1:t}).
\end{equation}
We will also assume that the oracle performs consistent deterministic tie-breaking: i.e. whenever two policies are tied, then it always outputs the same policy.
\end{defn}

\begin{algorithm}[tb]
   \caption{Follow the perturbed leader with fake sample perturbations - $\ftpl$.}
\label{defn:ftpl}
\begin{algorithmic}
	\FOR{each time step $t$}
   	\STATE Draw a random sequence of outcomes $\{z\}=(z^1,\ldots,\ldots,z^k)$ independently, based on some time-independent distribution over sequences. Both the length of the sequence and the outcome $z^i\in \Param$ at each iteration of the sequence can be random
	\STATE Denote with $\{z\} \cup \param^{1:t-1}$ the augmented sequence where we append the extra outcome samples $\{z\}$ at the beginning of sequence $\param^{1:t-1}$
	\STATE Invoke oracle $M$ and play policy:
\begin{equation}\label{eqn:opt-t}
\pi^t = \Or{\{z\}\cup \param^{1:{t-1}}}.
\end{equation}
	\ENDFOR
\end{algorithmic}
\end{algorithm}

In this generic setting, we define a new family of Follow-The-Perturbed-Leader (FTPL) algorithms where the perturbation takes the form of extra samples of outcomes (see Algorithm~\ref{defn:ftpl}).
In each round, the learning algorithm draws a random sequence of outcomes independently, and appends this sequence to the outcomes experienced during the learning process.
The algorithm invokes the oracle on this augmented outcome sequence, and plays the resulting policy.

\paragraph{Perturbed Leader Regret Analysis.}
We give a general theorem on the regret of a perturbed leader algorithm with sample perturbations. In the sections that follow we will give instances of this analysis in specific settings. 

\begin{theorem}\label{thm:general-regret}
For a distribution over sample sequences $\{z\}$ and a sequence of adversarially and adaptively chosen outcomes $\param^{1:T}$, define:
\begin{align*}
\stb &= \sum_{t=1}^T \E_{\{z\}}\left[\ell( \pi^t, \param^t)-\ell(\pi^{t+1},\param^t)\right]
\end{align*}
\begin{multline*}
\error = \E_{\{z\}}\left[\max_{\pi\in \Pi} \sum_{z^{\tau}\in \{z\}}\ell(\pi,z^\tau)\right]\\- \E_{\{z\}}\left[\min_{\pi\in \Pi}\sum_{z^{\tau}\in \{z\}}\ell(\pi,z^\tau)\right],
\end{multline*}
where $\pi^t$ is defined in Equation \eqref{eqn:opt-t}.
Then the expected regret of Algorithm~\ref{defn:ftpl} is upper bounded by,
\begin{equation}
\regret \leq \stb + \error.
\end{equation}
\end{theorem}

This theorem shows that any FTPL-variant where the perturbation can be described as a random sequence of outcomes has regret bounded by the two terms \stb~and \error.
Below we will instantiate this theorem to obtain concrete regret bounds for several problems.

The proof of the theorem is based on a well-known ``be-the-leader" argument. 
We first show that if we included the $t$th loss vector in the oracle call at round $t$, we would have regret bounded by \error, and then we show that the difference between our algorithm and this foreseeing one is bounded by \stb. 
See Appendix~\ref{app:oracles} for the proof.


\section{Adversarial Contextual Learning}\label{sec:cont-lin}

Our first specialization of the general setting is to \emph{contextual online combinatorial optimization}. 
In this learning setting, at each iteration, the learning algorithm picks a binary action vector $a^t \in \A\subseteq \{0,1\}^K$, from some feasibility set $\A$. We will interchangeably use $a^t$ both as a vector and as the set $\{j\in [K]:a^t(j)=1\}$.  The adversary picks a outcome $\param^t=(x^t,f^t)$ where $x^t$ belongs to some context space $\mathcal{X}$ and $f^t: \A \rightarrow \R$ is a cost function that maps each feasible action vector $a\in \A$ to a cost $f^t(a)$. The goal of the learning algorithm is to achieve low regret relative to a set of policies $\Pi \subset (\mathcal{X} \rightarrow \A)$ that map contexts to feasible action vectors. At each iteration the algorithm picks a policy $\pi^t$ and incurs a cost $\ell(\pi^t,\param^t) = f^t(\pi^t(x^t))$.
In this section, we consider the full-information problem, where after each round, the entire loss function $f^t$ is revealed to the learner. 
Online versions of a number of important learning tasks, including cost-sensitive classification, multi-label prediction, online linear optimization~\cite{Kalai2005} and online submodular minimization~\cite{Hazan2012} are all special cases of the contextual online combinatorial optimization problem, as we will see below.  

\paragraph{Contextual Follow the Perturbed Leader.} We will analyze the performance of an instantiation of the $\ftpl$ algorithm in this setting. 
To specialize the algorithm, we need only specify the distribution from which the sequence of fake outcomes $\{z\}$ is drawn at each time-step.
This distribution is parameterized by a subset of contexts $X\subseteq {\cal X}$, with $|X|=d$  and a noise parameter $\epsilon$.
We draw the sequence $\{z\}$ as follows: for each context $x \in X$, we add the fake sample $z_x=(x,f_x)$ where $f_x$ is a linear loss function based on a loss vector $\ell_x \in \R^K$, meaning that $f_x(a) = \dotp{a}{\ell_x}$. 
Each coordinate of the loss vector $\ell_x$ is drawn from a independent Laplace distribution with parameter $\epsilon$, i.e. for each coordinate $j\in [K]$ the density of $\ell_x(j)$ at $q$ is $f(q) = \frac{\epsilon}{2}\exp\{-\epsilon |q|\}$. The latter distribution has mean $0$ and variance $\frac{2}{\epsilon^2}$.
Using this distribution for fake samples gives an instantiation of Algorithm~\ref{defn:ftpl}, which we refer to as $\contlin(X,\epsilon)$ (see Algorithm~\ref{alg:cont-experts}).

\begin{algorithm}[tb]
  \caption{Contextual Follow the Perturbed Leader Algorithm - $\contlin(X,\epsilon)$.}
\label{alg:cont-experts}
\begin{algorithmic}
	\STATE {\bfseries Input:} parameter $\epsilon$, set of contexts $X$, policies $\Pi$.
	\FOR{each time step $t$}
   	\STATE Draw a sequence $\{z\} = (z_1,\ldots,z_d)$ of $d$ fake samples. 
	\STATE The context associated with sample $z_x$ is equal to $x$ and each coordinate of the loss vector $\ell_x$ is drawn i.i.d. from a Laplace$(\epsilon)$
	\STATE Pick and play according to policy 
\begin{equation}\label{eqn:opt-policy}
\pi^t = M(\{z\}\cup\param^{1:t-1})
\end{equation}
	\ENDFOR
\end{algorithmic}
\end{algorithm}

We analyze $\contlin(X, \epsilon)$ in two settings: the \emph{transductive setting} and the \emph{small separator} setting.

\begin{defn}
In the \emph{transductive setting}, at the beginning of the learning process, the adversary reveals to the learner the set of contexts that will arrive, although the ordering and multiplicity need not be revealed.
\end{defn}

\begin{defn}
In the \emph{small separator} setting, there exists a set $X \subset {\cal X}$ such that for any two distinct policies $\pi,\pi' \in \Pi$, there exists $x \in X$ such that $\pi(x) \ne \pi'(x)$. 
\end{defn}
In the transductive setting, the set $X$ that we use in $\contlin(X, \epsilon)$ is precisely this set of contexts that will arrive, which by assumption is available to the learning algorithm. 
In this small separator setting, the set $X$ used by $\contlin$ is the separating set.
This enables non-transductive learning, but one must be able to compute a small separator prior to learning.
Below we will see examples where this is possible. 


We now turn to bounding the regret of $\contlin(X,\epsilon)$.
Let $d=|X|$ be the number of contexts that are used in the definition of the noise distribution,  let $N = |\Pi| \le d^K$, and let $m$ denote the maximum number of non-zero coordinates that any policy can choose on any context, i.e. $m = \max_{a\in \A} \| a\|_1$. 
Even though at times we might constrain the sequence of loss functions that the adversary can pick (e.g. linear non-negative losses), we will assume that \emph{the oracle $M$ can handle at least linear loss functions with both positive and negative coordinates}.
Our main result is:
\begin{theorem}[Complete Information Regret]\label{thm:main-complete}
$\contlin(X,\epsilon)$ achieves regret against any adaptively and adversarially chosen sequence of contexts and loss functions:
\begin{enumerate}
\item\label{thm:it:trans} In the \emph{transductive} setting:
\begin{align*}
\regret ~\leq~& 
4 \epsilon  K\cdot \sum_{t=1}^T\E\left[\|f^t\|_*^2\right]+\frac{10}{\epsilon}\sqrt{dm}\log(N)
\end{align*}
\item In the \emph{transductive} setting, when loss functions are \emph{linear and non-negative}, i.e. $f^t(a) = \dotp{a}{\ell^t}$ with $\ell^t \in \R_{\geq 0}^K$:
\begin{align*}
\regret ~\leq~ 
\epsilon \cdot \sum_{t=1}^T\E\left[\dotp{\pi^t(x^t)}{\ell^t}^2\right]+\frac{10}{\epsilon}\sqrt{dm}\log(N)
\end{align*}
\item In the \emph{small separator} setting:
\begin{align*}
\regret ~\leq~ 
4\epsilon K d \cdot \sum_{t=1}^T \E\left[\|f^t\|_{*}^2\right]+\frac{10}{\epsilon}\sqrt{dm}\log(N)
\end{align*}
\end{enumerate}
where $\|f^t\|_* = \max_{a\in \A} |f^t(a)|$. 

When $\epsilon$ is set optimally, loss functions are in $[0,1]$, and loss vectors are in $[0,1]^K$, these give regret:\footnote{Observe that when loss vectors are in $[0,1]^K$, then the linear loss function is actually in $[0,m]$ not in $[0,1]$.} $O\left((dm)^{1/4}\sqrt{KT\log(N)}\right)$ in the first setting, $O\left(d^{1/4}m^{5/4} \sqrt{T\log(N)}\right)$ in the second and $O\left(m^{1/4}d^{3/4}\sqrt{KT\log(N)}\right)$ in the third. 
\end{theorem}

To prove the theorem we separately upper bound the $\stb$ and the $\error$ terms and then Theorem \ref{thm:main-complete} follows from Theorem \ref{thm:general-regret}.
One key step is a refined $\error$ analysis that leverages the symmetry of the Laplace distribution to obtain a bound with dependence $\sqrt{d}$ rather than $d$. 
This is possible only if the perturbation is centered about zero, and therefore does not apply to other FPTL variants that use non-negative distributions such as exponential or uniform~\cite{Kalai2005}.
Due to lack of space we defer proof details to Appendix~\ref{app:complete-info}.

This general theorem has implications for many specific settings that have been extensively studied in the literature.
We turn now to some examples.

\begin{example}(Transductive Contextual Experts)
The contextual experts problem is the online version of cost-sensitive multiclass classification, and the full-information version of the widely-studied contextual bandit problem.
The setting is as above, but $\mathcal{A}$ corresponds to sets with cardinality $1$, meaning that $m=1$ in our formulation. 
As a result, $\contlin$ can be applied as is, and the second claim in Theorem~\ref{thm:main-complete} shows that the algorithm has regret at most $O\left(d^{1/4}\sqrt{T\log(N)}\right)$ if at most $d$ contexts arrive. 
In the worst case this bound is $O(T^{3/4}\sqrt{\log(N)})$, since the adversary can choose at most $T$ contexts. 
To our knowledge, this is the first fully oracle-efficient algorithm for online adversarial cost-sensitive multiclass classification, albeit in the transductive setting.

This result can easily be lifted to infinite policy classes that have small Natarajan Dimension (a multi-class analog of VC-dimension), since such classes behave like finite ones once the set of contexts is fixed. 
Thus, in the transductive setting, Theorem~\ref{thm:main-complete} can be applied along with the analog of the Sauer-Shelah lemma, leading to a sublinear regret bound for classes with finite Natarajan dimension. 
On the other hand, in the non-transductive case it is possible to construct examples where achieving sublinear regret against a VC class is information-theoretically hard, demonstrating a significant difference between the two settings. 
See Corollary~\ref{cor:natarajan} and Theorem~\ref{thm:hardness} in the Appendix~\ref{sec:infinite} for details on these arguments. \ed
\end{example}

\begin{example}(Non-contextual Shortest Path Routing and Linear Optimization)
For the case when the linear optimization corresponds to computing the shortest $(s,t)$-path in a DAG, then $K$ and $m$ equal to the number of edges and the problem can be solved in poly-time even when edge costs are negative.
More generally, $\contlin$ can also be applied to non-contextual problems, which is a special case where $d=1$. 
In such a case, $\contlin$ reduces to the classical FTPL algorithm with Laplace instead of Exponential noise, and Theorem~\ref{thm:main-complete} matches existing results for online linear optimization~\cite{Kalai2005}.
In particular, for problems without context, $\contlin$ has regret that scales with $\sqrt{T}$. \ed
\end{example}

\begin{example}(Online sub-modular minimization)
A special case of our setting is the online-submodular minimization problem studied in previous work~\cite{Hazan2012,jegelka2011online}.
As above, this is a non-contextual online combinatorial optimization problem, where the loss function $f^t$ presented at each round is submodular.
Here, $\contlin$ reduces to the strongly polynomial algorithm of~\citet{Hazan2012}, although our noise follows a Laplace instead of Uniform distribution.
A straightforward application of the first claim of Theorem~\ref{thm:main-complete} shows that $\contlin$ achieves regret at most $O(KH\sqrt{T\log(K)})$ if the losses are bounded in $[-H,H]$, and a slightly refined analysis of the error terms gives $O(KH\sqrt{T})$ regret. 
This matches the FTPL analysis of~\citet{Hazan2012}, although they also develop an algorithm based on online convex optimization that achieves $O(H\sqrt{KT})$ regret. \ed
\end{example}

\begin{example}(Contextual Experts with linear policy classes)
The third clause of Theorem~\ref{thm:main-complete} gives strong guarantees for the non-transductive contextual experts problem, provided one can construct a small separating set of contexts.
Often this is possible, and we provide some examples here.
\begin{enumerate}
\item For binary classification where the policies are boolean disjunctions (conjunctions) over $n$ binary variables, the set of $1$-sparse ($n-1$-sparse) boolean vectors form a separator of size $n$. 
This is easy to see as two disjunctions must disagree on at least one variable, so they will make different predictions on the vector that is non-zero only in that component.
Note that the size of the small separator is independent of the time horizon $T$ and logarithmic in the number of policies.
Thus, Theorem~\ref{thm:main-complete} shows that $\contlin$ suffers at most $O(\sqrt{T}\log(N))$ regret since $d=\log(N), m=1$ and $K=2$. 
\item For binary classification in $n$ dimensions, consider a discretization of linear classifiers defined as follows, the separating hyperplane of each classifier is defined by choosing the intercept with each axis from one of $O(1/\tau)$ values (possibly including something denoting no intercept). 
Then a small separator includes, for each axis, one point between each pair in the discretization, for a total of $O(n/\tau)$ points. 
This follows since any two distinct classifiers have different intercepts for at least one axis, and our small separator has one point between these two different intercepts, leading to different predictions.
Note that the number of classifiers in the discretization is $O(\tau^{-n})$.
Here Theorem~\ref{thm:main-complete} shows that $\contlin$ suffers at most $O(\frac{n\sqrt{T}}{\tau^{3/4}}(\log(\frac{1}{\tau}))^{1/4})$ regret since $N = O(\tau^{-n}), d=\frac{n}{\tau}, m=1$ and $K=2$. 
This bound has a undesireable polynomial dependence on the discretization resolution $\tau$ but avoids exponential dimension dependence. 
\end{enumerate}
Thus we believe that the smallest separator size for a policy class can be viewed as a new complexity measure, which may be of independent interest. \ed
\end{example}

\section{Linear Losses and Semi-Bandit Feedback}
\label{sec:contextual-bandits}

In this section, we consider contextual learning with semi-bandit feedback and linear non-negative losses.
At each round $t$ of this learning problem, the adversary chooses a non-negative vector $\ell^t \in \mathbb{R}^K_{\geq 0}$ and sets the loss function to $f^t(a) = \dotp{a}{\ell^t}$.
The learner chooses an action $a^t \in \mathcal{A} \subset \{0,1\}^K$ accumulates loss $f^t(a^t)$ and observes $\ell^t(j)$ for each $j \in a^t$.
In other words, the learner observes the coefficients for only the elements in the set that he picked. 
Notice that if $\mathcal{A}$ is the one-sparse vectors, then this setting is equivalent to the well-studied contextual bandit problem~\cite{langford2008epoch}.

\paragraph{Semi-bandit algorithm.}
Our semi-bandit algorithm proceeds as follows: At each iteration it makes a call to $\contlin(\epsilon)$, which returns a policy $\pi^t$ and implies a chosen action $a^t=\pi^t(x^t)$. 
The algorithm plays the action $a^t$, observes the coordinates of the loss $\{\ell^t(j)\}_{j \in a^t}$ and proceeds to construct an \emph{proxy loss vector} $\hat{\ell}^t$, which it passes to the instance of $\contlin$, before proceeding to the next round.

To describe the construction of $\hat{\ell}^t$, let $p^t(\pi)=\Pr[\pi^t=\pi|\mathcal{H}^{t-1}]$ denote the probability that $\contlin$ returns policy $\pi$ at time-step $t$ conditioned on the past history (observed losses and contexts, chosen actions, current iteration's context, internal randomness etc., which we denote with $\mathcal{H}^{t-1}$).
For any element $j\in [K]$, let:
\begin{equation}
q^t(j) = \sum_{\pi\in \Pi: j\in \pi(x^t)} p^t(\pi)
\end{equation}
denote the probability that element $j$ is included in the action chosen by $\contlin(X,\epsilon)$ at time-step $t$.

Typical semi-bandit algorithms aim to construct proxy loss vectors by dividing the observed coordinates of the loss by the probabilities $q^t(j)$ and setting other coordinates to zero, which is the well-known inverse propensity scoring mechanism~\cite{HorvitzTh52}. 
Unfortunately, in our case, the probabilities $q^t(j)$ stem from randomness fed into the oracle, so that they are implicit maintained and therefore must be approximated.

We therefore construct $\hat{\ell}^t$ through a geometric sampling scheme due to~\citet{neu2013efficient}.
For each $j\in \pi^t(x^t)$, we repeatedly invoke the current execution of the $\contlin$ algorithm with fresh noise, until it returns a policy that includes $j$ in its action for context $x^t$.
The process is repeated at most $L$ times for each $j\in \pi^t(x^t)$ and the number of invocations is denoted $J^t(j)$.
The vector $\hat{\ell}^t$ that is returned to the full feedback algorithm is zero for all $j \notin \pi^t(x^t)$, and for each $j \in \pi^t(x^t)$ it is $\hat{\ell}^t(j)=J^t(j)\cdot \ell^t(j)$.

By Lemma 1 of~\citet{neu2013efficient}, this process yields a proxy loss vector $\hat{\ell}^t$ that satisfies,
\begin{equation}\label{eqn:bias}
\E\left[\hat{\ell}^t(j)~|~\F^{t-1}\right]= \left(1-\left(1-q^t(j)\right)^L\right)\ell^t(j).
\end{equation} 
The semi-bandit algorithm feeds this proxy loss vector to the $\contlin$ instance and proceeds to the next round.


\begin{algorithm}[tb]
\caption{Contextual Semi-Bandit Algorithm - $\contsemband(X,\epsilon,L)$.}
\label{alg:cont-semi-bandits}
\begin{algorithmic}
	\STATE {\bfseries Input:} parameter $\epsilon,M$, set of contexts $X$, policies $\Pi$.
	\STATE Let $D$ denote a distribution over a sequence of $d$ samples, $\{z\} = (z_1,\ldots,z_d)$, where the context associated with sample $z_x$ is equal to $x$ and each coordinate of the loss vector $\ell_x$ is drawn i.i.d. from a Laplace$(\epsilon)$\\
\FOR{each time-step $t$}
\STATE Draw a sequence $\{z\}^t$ from distribution $D$.
\STATE Pick and play according to policy 
\begin{equation}\label{eqn:opt-policy2}
\pi^t = M(\{z\}\cup(x^{1:t-1},\hat{\ell}^{1:t-1}))
\end{equation}
\STATE Observe loss $\ell^t(j)$ for each $j\in \pi^t(x^t)$
\STATE Set $\hat{\ell}^t(j)=0$ for any $j\notin \pi^t(x^t)$
\STATE Set $\hat{\ell}^t(j) = J^t(j)\cdot\ell^t(j)$, for each $j\in \pi^t(x^t)$, where $J^t(j)$ is computed by the following geometric sampling process:
\FOR{each element $j\in \pi^t(x^t)$}
	\FOR{each iteration $i=1,\ldots,L$}
		\STATE Draw a sequence $\{y\}^i$ from distribution $D$.
		\STATE Compute $\pi^i= M(\{y\}^i\cup(x^{1:t-1},\hat{\ell}^{1:t-1}))$
		\STATE If $j\in\pi^i(x^t)$ then stop and return $J^t(j)=i$
	\ENDFOR
\ENDFOR
\STATE If process finished without setting $J^t(j)$, then set $J^t(j)=L$
\ENDFOR
\end{algorithmic}
\end{algorithm}


The formal description of the complete bandit algorithm is given in Algorithm \ref{alg:cont-semi-bandits} and we refer to it as $\contsemband(X,\epsilon,L)$. We bound its regret in the transductive and small separator setting.

\begin{theorem}\label{thm:semi-bandit-regret}
The expected regret of $\contsemband(X,\epsilon,L)$ in the semi-bandit setting against any adaptively and adversarially chosen sequence of contexts and linear non-negative losses, with $\|\ell^t\|_*\leq 1$, is at most:
\begin{itemize}
\item In the transductive setting:
\begin{multline*}
\regret \leq 
2\epsilon m K T+ \frac{10}{\epsilon}\sqrt{dm}\log(N)+\frac{K T}{eL}
\end{multline*}

\item In the small separator setting: 
\begin{multline*}
\regret \leq 
8\epsilon K^2 d L m T+ \frac{10}{\epsilon}\sqrt{dm}\log(N)+\frac{KT}{eL}
\end{multline*}
\end{itemize}
For $L=\sqrt{KT}$ and optimal $\epsilon$, the regret is $O\left(d^{1/4}m^{3/4}\sqrt{KT\log(N)}\right)$ in the first setting. 
For $L=T^{1/3}$ and optimal $\epsilon$, the regret is $O\left((md)^{3/4} K T^{2/3}\sqrt{\log(N)}\right)$ in the second setting.
Moreover, each iteration of the algorithm requires $mL$ oracle calls and otherwise runs in polynomial time in $d, K$.
\end{theorem}

This is our main result for adversarial variants of the contextual bandit problem. 
In the most well-studied setting, i.e. contextual bandits, we have $m=1$, so our regret bound is $O(d^{1/4}\sqrt{KT\log(N)})$ in the transductive setting and $O(d^{3/4}KT^{2/3}\sqrt{\log(N)})$ in the small separator setting.
Since for the transductive case $d \le T$ and for the small-separator case $d$ can be independent of $T$ (see discussion above), this implies sublinear regret for adversarial contextual bandits in either setting. 
To our knowledge this is the first oracle-efficient sublinear regret algorithm for variants of the contextual bandit problem. 
However, as we mentioned before, neither regret bound matches the optimal $O(\sqrt{KT\log(N)})$ rate for this problem, which can be achieved by computationally intractable algorithms.
An interesting open question is to develop computationally efficient, statistically optimal contextual bandit algorithms. 





\section{Switching Policy Regret}\label{sec:switching}

In this section we analyze switching regret for the contextual linear optimization setting, i.e. regret that compares to the best sequence of policies that switches at most $k$ times. Such a notion of regret was first analyzed by~\citet{Herbster1998} and several algorithms, that are not computationally efficient for large policy spaces, have been designed since then (e.g. \cite{Luo2015}). Our results provide the first computationally efficient switching regret algorithms assuming offline oracle access. 

For this setting we will assume that \emph{the learner knows the exact sequence $x^{1:T}$ of contexts ahead of time and not only the set of potential contexts.} The extension stems from the realization that we can simply think of time $t$ as part of the context at time-step $t$. Thus now the contexts are of the form $\tilde{x}^t=(t,x^t)$. Moreover, policies in the augmented context space are now of the form: $\tilde{\pi}(\tilde{x}^t) = \pi_{I(t)}(x^t)$, where $I(t)$ is a selector which maps a time-step $t$ to a policy $\pi\in \Pi$, with the constraint that the number of time-steps such that $I(t)\neq I(t-1)$ is at most $k$. If the original policy space $\Pi$ was of size $N$, the new policy space, denoted $\tilde{\Pi}$, is of size $\tilde{N}$ at most $T^kN^k$, since there are at most $T^k$ partitions of time into $k$ consequetive intervals and each of the $k$ intervals can be occupied by $N$ possible policies.  Moreover, in this augmented context space, the number of possible contexts, denoted $\tilde{X}$ is equal to $\tilde{d}=T$. 

Thus if we run $\contlin(X,\epsilon)$ on this augmented context and policy space, Theorem \ref{thm:main-complete}, bounds the regret against all policies in the augmented policy space $\tilde{\Pi}$.
Since, regret against the augmented policy space, corresponds to switching regret against the original set of policies, the following corollary is immediate:
\begin{corollary}[Contextual Switching Regret]\label{thm:contextual-switching-regret}
In the transductive complete information setting, $\contlin(\tilde{X},\epsilon)$ applied to the augmented policy space $\tilde{\Pi}$, achieves $k$-switching regret against any adaptively and adversarially chosen sequence of contexts and losses at most:
$O\left(m^{1/4}\sqrt{K k\log(TN)}T^{3/4}\right)$ for general loss functions in $[0,1]$ and $O\left(\sqrt{k\log(TN)}m^{5/4} T^{3/4}\right)$ for linear losses with loss vectors in $[0,1]^{K}$.
\end{corollary}

It remains to show is that we can efficiently solve the offline optimization problem for the new policy space $\tilde{\Pi}$, if we have access to an optimization oracle for the original policy space $\Pi$. Then we can claim that $\contlin(\tilde{X},\epsilon)$ in the augmented context and policy space is also an efficient algorithm. We show that the latter is true via a dynamic programming approach. The approach generalizes beyond contextual linear optimization settings.

\begin{lemma}
The oracle $\tilde{M}$ in the augmented space, 
\begin{equation}
\tilde{M}(\tilde{\param}^{1:T}) = \arginf_{\tilde{\pi}\in \tilde{\Pi}} \sum_{\tau=1}^T \dotp{\tilde{\pi}(\tau,x_\tau)}{\ell^\tau}
\end{equation}
is computable in $O(T k)$ time, with $O(T^2)$ calls to the oracle over the original space, $M$.
This process can be amortized so that solving a sequence of $T$ problems in the augmented space requires $O(T^2)$ calls to $M$ in total.
\end{lemma}
\begin{proof}
Oracle $\tilde{M}$ must compute the best sequence of policies $\pi^1,\ldots, \pi^T$, such that $\pi^t\neq \pi^{t-1}$ at most $k$ times. Let $R(t,q)$ denote the loss of the optimal sequence of policies up to time-step $t$ and with at most $q$ switches. Then it is easy to see that:
\begin{equation}
R(t,q) = \min_{\tau\leq t} R(\tau,q-1) + \Ell\left(M(\param^{\tau+1:t}),\param^{\tau+1:t}\right),
\end{equation}
i.e. compute the best sequence of policies up till some time step $\tau\leq t$ with at most $q-1$ switches and then augment it with the optimal fixed policy for the period $(\tau+1,t)$. Then take the best over possible times $\tau\leq t$.

This can be implemented by first invoking oracle $M$ for every possible period $[\tau_1,\tau_2]$. Then filling up iteratively all the entries $R(t,q)$. For $q=0$, the problem $R(t,0)$ corresponds to exactly the original oracle problem $M$, hence for each $t$, we can solve the problem $R(t,0)$. Computing all values of $R(t,q)$ then takes time $T k$ in total.
\end{proof}

\begin{example}(Efficient switching regret for non-contextual problems)
When the original space has no contexts, our result above implies the first efficient sub-linear switching regret algorithm for online linear optimization. 
In this case, the transductivity assumption is trivially satisfied as there is no contextual information, and our the instance of $\contlin$ runs on a sequence of contexts that just encode time. 
One concrete example where linear optimization with both positive and negative weights is polynomially solvable is the online shortest path problem on a directed acyclic graph.
Our result implies a fully efficient, sublinear switching regret algorithm for the online shortest-path problem on a DAG, and our algorithm performs $t$ shortest-path computations at the $t$th iteration. 
The result also covers other examples, such as online matroid optimization. \ed
\end{example}

\section{Efficient Path Length Regret Bounds}\label{sec:path-length}

In this section we examine a variant of our $\contalg(\epsilon)$ algorithm that is efficient and achieves regret that is upper bounded by structural properties of the utility sequence. Our algorithm is framed in terms of a generic predictor that the learner has access to and the regret is upper bounded by the deviation of the true loss vector from the predictor. For specific instances of the predictor this leads to path length bounds \cite{Chiang2012} or variance based bounds \cite{Hazan2010}. Our approach is general enough to allow for generalizations of variance and path length that can incorportate contextual information and can be viewed as an efficient version and a generalization of the results of~\citet{Rakhlin2013a} on learning with predictable sequences. Such results have also found applications in learning in game theoretic environments \cite{Rakhlin2013,syrgkanis2015fast}.

The algorithm is identical to $\contlin(\epsilon)$ with the exception that now the policy that is used at time-step $t$ is:
\begin{equation}
\pi^t = M(\{z\}\cup \param^{1:t-1}\cup(x^t,\pred^t))
\end{equation}
where $\pred^t\in \{0,1\}^K\rightarrow \R^K$ is an arbitrary loss function predictor, which can depend on the observed history up to time $t$. This predictor can be interpreted as partial side information that the learner has about the loss function that will arrive at time-step $t$. Given such a predictor we can define the error between the predictor and the actual sequence:
\begin{equation}
{\cal E}^t=\E\left[ \|f^t-\pred^t\|_{*}^2\right]
\end{equation}

\begin{theorem}[Predictor based regret bounds]\label{thm:oftpl-bound}
The regret of $\contalg(X,\epsilon)$ with predictors and complete information, 
\begin{enumerate}
\item In the transductive setting is upper bounded by: \begin{equation*}
\vspace{-1em}\hspace{-1em}\regret \leq 
4\epsilon K \sum_{t=1}^T {\cal E}^t+\frac{10\sqrt{dm}\log(N)}{\epsilon}
\end{equation*}
\item In the small separator setting is upper bounded by:
\begin{equation*}
\vspace{-.7em}\hspace{-1em}\regret \leq 
4\epsilon K d \sum_{t=1}^T {\cal E}^t +\frac{10\sqrt{dm}\log(N)}{\epsilon}
\end{equation*}
\end{enumerate}
Picking $\epsilon$ optimally gives regret
$O\left((dm)^{1/4}\sqrt{K\log(N)\sum_{t=1}^T {\cal E}^t}\right)$ in the first setting and 
$O\left(m^{1/4}d^{3/4}\sqrt{K\log(N)\sum_{t=1}^T {\cal E}^t}\right)$ in the second.
\end{theorem}

Even without contexts, our result is the first efficient path length regret algorithm for online combinatorial optimization. For instance, for the case of non-contextual, online combinatorial optimization an instantiation of our algorithm achieves regret $O\left( m^{1/4} \sqrt{K\log(K) \sum_{t=1}^T{\cal E}^t}\right)$ against adaptive adversaries.
For learning with expert, $m=1$ and $K$ is number of experts, the results of \citet{Rakhlin2013a} provide a non-efficient $O\left(\sqrt{\log(K) \sum_{t=1}^T{\cal E}^t}\right)$. Thus our bound incurs an extra cost of $\sqrt{K}$ in comparison. Removing this extra factor of $\sqrt{K}$ in an efficient manner is an interesting open question.

\section{Discussion}
In this work we give fully oracle efficient algorithms for adversarial online learning problems including contextual experts, contextual bandits, and problems involving linear optimization or switching experts.
Our main algorithmic contribution is a new Follow-The-Perturbed-Leader style algorithm that adds perturbed low-dimensional statistics.
We give a refined analysis for this algorithm that guarantees sublinear regret for all of these problems. 
All of our results hold against adaptive adversaries, both with full and partial feedback. 

While our algorithms achieve sublinear regret in all problems we consider, we do not always match the regret bounds attainable by inefficient alternatives. 
An interesting direction for future work is whether fully oracle-based algorithms can achieve optimal regret bounds in the settings we consider.
Another interesting direction focuses on a deeper understanding of the small-separator condition and whether it enables efficient non-transductive learning in other settings.
We look forward to studying these questions in future work.

\bibliographystyle{icml2016}
\bibliography{poa_survey}

\begin{thebibliography}{26}
\providecommand{\natexlab}[1]{#1}
\providecommand{\url}[1]{\texttt{#1}}
\expandafter\ifx\csname urlstyle\endcsname\relax
  \providecommand{\doi}[1]{doi: #1}\else
  \providecommand{\doi}{doi: \begingroup \urlstyle{rm}\Url}\fi

\bibitem[Agarwal et~al.(2014)Agarwal, Hsu, Kale, Langford, Li, and
  Schapire]{Agarwal2014}
Agarwal, Alekh, Hsu, Daniel, Kale, Satyen, Langford, John, Li, Lihong, and
  Schapire, Robert~E.
\newblock Taming the monster: {A} fast and simple algorithm for contextual
  bandits.
\newblock In \emph{International Conference on Machine Learning (ICML)}, 2014.

\bibitem[Auer et~al.(1995)Auer, Cesa-Bianchi, Freund, and Schapire]{Auer1995}
Auer, Peter, Cesa-Bianchi, Nicolo, Freund, Yoav, and Schapire, Robert~E.
\newblock Gambling in a rigged casino: The adversarial multi-armed bandit
  pproblem.
\newblock In \emph{Foundations of Computer Science (FOCS)}, 1995.

\bibitem[Awerbuch \& Kleinberg(2008)Awerbuch and Kleinberg]{awerbuch2008online}
Awerbuch, Baruch and Kleinberg, Robert.
\newblock Online linear optimization and adaptive routing.
\newblock \emph{Journal of Computer and System Sciences}, 2008.

\bibitem[Ben-David et~al.(1995)Ben-David, Cesa-Bianchi, Haussler, and
  Long]{bendavid1995characterizations}
Ben-David, Shai, Cesa-Bianchi, Nicolo, Haussler, David, and Long, Philip~M.
\newblock Characterizations of learnability for classes of (0,..., n)-valued
  functions.
\newblock \emph{Journal of Computer and System Sciences}, 1995.

\bibitem[Ben-David et~al.(1997)Ben-David, Kushilevitz, and
  Mansour]{ben1997online}
Ben-David, Shai, Kushilevitz, Eyal, and Mansour, Yishay.
\newblock Online learning versus offline learning.
\newblock \emph{Machine Learning}, 1997.

\bibitem[Cesa-Bianchi \& Shamir(2011)Cesa-Bianchi and
  Shamir]{cesa2011efficient}
Cesa-Bianchi, Nicolo and Shamir, Ohad.
\newblock Efficient online learning via randomized rounding.
\newblock In \emph{Advances in Neural Information Processing Systems (NIPS)},
  2011.

\bibitem[Cesa-Bianchi et~al.(1997)Cesa-Bianchi, Freund, Haussler, Helmbold,
  Schapire, and Warmuth]{cesa1997use}
Cesa-Bianchi, Nicolo, Freund, Yoav, Haussler, David, Helmbold, David~P,
  Schapire, Robert~E, and Warmuth, Manfred~K.
\newblock How to use expert advice.
\newblock \emph{Journal of the ACM (JACM)}, 1997.

\bibitem[Chiang et~al.(2012)Chiang, Yang, Lee, Mahdavi, Lu, Jin, and
  Zhu]{Chiang2012}
Chiang, Chao-Kai, Yang, Tianbao, Lee, Chia-Jung, Mahdavi, Mehrdad, Lu, Chi-Jen,
  Jin, Rong, and Zhu, Shenghuo.
\newblock Online optimization with gradual variations.
\newblock In \emph{Conference on Learning Theory (COLT)}, 2012.

\bibitem[Daskalakis \& Syrgkanis(2015)Daskalakis and Syrgkanis]{Daskalakis2015}
Daskalakis, Constantinos and Syrgkanis, Vasilis.
\newblock Learning in auctions: Regret is hard, envy is easy.
\newblock \emph{arXiv:1511.01411}, 2015.

\bibitem[Dud{\'i}k et~al.(2011)Dud{\'i}k, Hsu, Kale, Karampatziakis, Langford,
  Reyzin, and Zhang]{dudik2011efficient}
Dud{\'i}k, Miroslav, Hsu, Daniel, Kale, Satyen, Karampatziakis, Nikos,
  Langford, John, Reyzin, Lev, and Zhang, Tong.
\newblock Efficient optimal learning for contextual bandits.
\newblock In \emph{Uncertainty and Artificial Intelligence (UAI)}, 2011.

\bibitem[Freund \& Schapire(1997)Freund and Schapire]{Freund1997}
Freund, Yoav and Schapire, Robert~E.
\newblock A decision-theoretic generalization of on-line learning and an
  application to boosting.
\newblock \emph{Journal of Computer and System Sciences}, 1997.

\bibitem[Haussler \& Long(1995)Haussler and Long]{haussler1995generalization}
Haussler, David and Long, Philip~M.
\newblock A generalization of sauer's lemma.
\newblock \emph{Journal of Combinatorial Theory}, 1995.

\bibitem[Hazan \& Kale(2010)Hazan and Kale]{Hazan2010}
Hazan, Elad and Kale, Satyen.
\newblock Extracting certainty from uncertainty: regret bounded by variation in
  costs.
\newblock \emph{Machine Learning}, 2010.

\bibitem[Hazan \& Kale(2012)Hazan and Kale]{Hazan2012}
Hazan, Elad and Kale, Satyen.
\newblock Online submodular minimization.
\newblock \emph{Journal of Machine Learning Research (JMLR)}, 2012.

\bibitem[Herbster \& Warmuth(1998)Herbster and Warmuth]{Herbster1998}
Herbster, Mark and Warmuth, Manfred~K.
\newblock Tracking the best expert.
\newblock \emph{Machine Learning}, 1998.

\bibitem[Horvitz \& Thompson(1952)Horvitz and Thompson]{HorvitzTh52}
Horvitz, Daniel~G and Thompson, Donovan~J.
\newblock A generalization of sampling without replacement from a finite
  universe.
\newblock \emph{Journal of the American Statistical Association (JASA)}, 1952.

\bibitem[Hutter \& Poland(2005)Hutter and Poland]{Hutter2005}
Hutter, Marcus and Poland, Jan.
\newblock Adaptive online prediction by following the perturbed leader.
\newblock \emph{Journal of Machine Learning Research (JMLR)}, 2005.

\bibitem[Jegelka \& Bilmes(2011)Jegelka and Bilmes]{jegelka2011online}
Jegelka, Stefanie and Bilmes, Jeff~A.
\newblock Online submodular minimization for combinatorial structures.
\newblock In \emph{International Conference on Machine Learning (ICML)}, 2011.

\bibitem[Kakade \& Kalai(2005)Kakade and Kalai]{kakade2005batch}
Kakade, Sham~M and Kalai, Adam.
\newblock From batch to transductive online learning.
\newblock In \emph{Advances in Neural Information Processing Systems (NIPS)},
  2005.

\bibitem[Kalai \& Vempala(2005)Kalai and Vempala]{Kalai2005}
Kalai, Adam and Vempala, Santosh.
\newblock Efficient algorithms for online decision problems.
\newblock \emph{Journal of Computer and System Sciences}, 2005.

\bibitem[Langford \& Zhang(2008)Langford and Zhang]{langford2008epoch}
Langford, John and Zhang, Tong.
\newblock The epoch-greedy algorithm for multi-armed bandits with side
  information.
\newblock In \emph{Advances in Neural Information Processing Systems (NIPS)},
  2008.

\bibitem[Luo \& Schapire(2015)Luo and Schapire]{Luo2015}
Luo, Haipeng and Schapire, Robert~E.
\newblock Achieving all with no parameters: Adanormalhedge.
\newblock In \emph{Conference on Learning Theory (COLT)}, 2015.

\bibitem[Neu \& Bart{\'o}k(2013)Neu and Bart{\'o}k]{neu2013efficient}
Neu, Gergely and Bart{\'o}k, G{\'a}bor.
\newblock An efficient algorithm for learning with semi-bandit feedback.
\newblock In \emph{Algorithmic Learning Theory (ALT)}, 2013.

\bibitem[Rakhlin \& Sridharan(2013{\natexlab{a}})Rakhlin and
  Sridharan]{Rakhlin2013}
Rakhlin, Alexander and Sridharan, Karthik.
\newblock Optimization, learning, and games with predictable sequences.
\newblock In \emph{Advances in Neural Information Processing Systems (NIPS)},
  pp.\  3066--3074, 2013{\natexlab{a}}.

\bibitem[Rakhlin \& Sridharan(2013{\natexlab{b}})Rakhlin and
  Sridharan]{Rakhlin2013a}
Rakhlin, Alexander and Sridharan, Karthik.
\newblock Online learning with predictable sequences.
\newblock In \emph{Conference on Learning Theorem (COLT)}, 2013{\natexlab{b}}.

\bibitem[Syrgkanis et~al.(2015)Syrgkanis, Agarwal, Luo, and
  Schapire]{syrgkanis2015fast}
Syrgkanis, Vasilis, Agarwal, Alekh, Luo, Haipeng, and Schapire, Robert~E.
\newblock Fast convergence of regularized learning in games.
\newblock In \emph{Advances in Neural Information Processing Systems (NIPS)},
  2015.

\end{thebibliography}

\appendix
\onecolumn

\begin{appendix}

\begin{center}
\bf \Large Supplementary material for \\ ``Efficient Algorithms for Adversarial Contextual Learning''
\end{center}

\section{Omitted Proofs from Section~\ref{sec:oracles}}
\label{app:oracles}

\subsection{Proof of Theorem~\ref{thm:general-regret}}


We prove the theorem by analyzing a slightly modified algorithm, that only draws the perturbation once at the beginning of the learning process but is otherwise identical.
The bulk of the proof is devoted to bounding this modified algorithm's regret against oblivious adversaries, i.e. an adversary that chooses the outcomes $\param^{1:T}$ before the learning process begins. 
We use this regret bound along with a reduction due to Hutter and Poland~\cite{Hutter2005} (see their Lemma 12) to obtain a regret bound for Algorithm~\ref{defn:ftpl} against adaptive adversaries.
We provide a proof of this reduction in Appendix~\ref{sec:app-oblivious} and proceed here with the analysis of the modified algorithm.

To bound the regret of the modified algorithm, consider letting the algorithm observe $\param^t$ ahead of time, so that at each time step $t$, the algorithm plays $\pi^{t+1}=\Or{\{z\}\cup\param^{1:t}}$. 
Notice trivially that the regret of the modified algorithm is,
\begin{align*}
\regret = \sum_{t=1}^T\ell(\pi^t,\param^t) - \min_{\pi \in \Pi} \ell(\pi^*,\param^t) = \sum_{t=1}^T\ell(\pi^t,\param^t) - \ell(\pi^{t+1},\param^t) + \sum_{t=1}^T\ell(\pi^{t+1},\param^t) - \min_{\pi \in \Pi}\ell(\pi^*,\param^t) 
\end{align*}
The first sum here is precisely the \stb~term in the bound, so we must show that the second sum is bounded by \error.
This is proved by induction in the following lemma.


\begin{lemma}[Be-the-leader with fixed sample perturbations]\label{lem:pbtl} For any realization of the sample sequence $\{z\}$ and for any policy $\pi^*$:
\begin{equation}
\sum_{t=1}^T \left( \ell(\pi^{t+1},\param^t)-\ell(\pi^*,\param^t) \right) \leq   \max_{\pi\in \Pi} \sum_{z^{\tau}\in \{z\}}\ell(\pi,z^\tau)- \min_{\pi\in \Pi}\sum_{z^{\tau}\in \{z\}}\ell(\pi,z^\tau)
\end{equation}
\end{lemma}
\begin{proof}
Denote with $k$ the length of sequence $\{z\}$. Consider the sequence $\{z\}\cup \param^{1:T}$ and let $a^1= M(\{z\})$. We will show that for any policy $\pi^*$:
\begin{align}\label{eqn:induction}
\sum_{\tau=1}^{k} \ell(\pi^1,z^\tau) + \sum_{t=1}^T \ell(\pi^{t+1},\param^t) \leq 
\sum_{\tau=1}^{k} \ell(\pi^*,z^\tau) + \sum_{t=1}^T \ell(\pi^*,\param^t)
\end{align}
For $T=0$, the latter trivially holds by the definition of $a^1$. Suppose it holds for some $T$, we will show that it holds for $T+1$. Since the induction hypothesis holds for any $\pi^*$, applying it for $a^{T+2}$, i.e.,:
\begin{align*}
\sum_{\tau=1}^{k} \ell(\pi^1,z^\tau) + \sum_{t=1}^{T+1} \ell(\pi^{t+1},\param^t) \leq~&
\sum_{\tau=1}^{k} \ell(\pi^{T+2},z^\tau) + \sum_{t=1}^T \ell(\pi^{T+2},\param^t)+ \ell(\pi^{T+2},\param^{T+1})\\
=~&
\sum_{\tau=1}^{k} \ell(\pi^{T+2},z^\tau) + \sum_{t=1}^{T+1} \ell(\pi^{T+2},\param^t)
\end{align*}
By definition of $a^{T+2}$ the latter is at most: $
\sum_{\tau=1}^{k} \ell(\pi^*,z^\tau) + \sum_{t=1}^{T+1} \ell(\pi^*,\param^t)$ for any $\pi^*$. Which proves the induction step. Thus, by re-arranging Equation \eqref{eqn:induction} we get:
\begin{align*}
\sum_{t=1}^T \left( \ell(\pi^{t+1},\param^t)-\ell(\pi^*,\param^t)\right) ~\leq~& \sum_{\tau=1}^k \left( \ell(\pi^*,z^\tau)-\ell(\pi^1,z^\tau)\right)
 ~\leq~  \max_{\pi\in \Pi} \sum_{\tau=1}^k \ell(\pi,z^{\tau})-\min_{\pi\in \Pi} \sum_{\tau=1}^k \ell(\pi,z^\tau)
\end{align*}
\end{proof}

Thus the regret of the modified algorithm against an oblivious adversary is bounded by $\stb + \error$.
By applying the reduction of Hutter and Poland~\cite{Hutter2005} (see Appendix~\ref{sec:app-oblivious} for a proof sketch), the regret of Algorithm~\ref{defn:ftpl} is bounded is bounded in the same way.

\subsection{From adaptive to oblivious adversaries}\label{sec:app-oblivious} 
We will utilize a generic reduction provided in Lemma 12 of \cite{Hutter2005}, which states that given that in Algorithm~\ref{defn:ftpl} we draw independent randomization at each iteration, it suffices to provide a regret bound only for oblivious adversaries, i.e., the adversary picks a fixed sequence $\param^{1:T}$ ahead of time without observing the policies of the player. Moreover, for any such fixed sequence of an oblivious adversary, the expected utility of the algorithm can be easily shown to be equal to the expected utility if we draw a single random sequence $\{z\}$ ahead of time and use the same random vector all the time. 

The proof is as follows: by linearity of expectation and the fact that each sequence $\{z\}^t$ drawn at each time-step $t$ is identically distributed:
\begin{align*}
\E_{\{z\}^1,\ldots,\{z\}^t}\left[\sum_{t=1}^T u(M(\{z\}^t\cup \param^{1:t-1}),\param^t)\right] =~& \sum_{t=1}^T \E_{\{z\}^t}\left[u(M(\{z\}^t\cup \param^{1:t-1}),\param^t)\right]\\
 =~& \sum_{t=1}^T \E_{\{z\}^1}\left[u(M(\{z\}^1\cup \param^{1:t-1}),\param^t)\right]\\
=&\E_{\{z\}^1}\left[\sum_{t=1}^T u(M(\{z\}^1\cup \param^{1:t-1}),\param^t)\right]
\end{align*}
The latter is equivalent to the expected reward if we draw a single random sequence $\{z\}$ ahead of time and use the same random vector all the time. Thus it is sufficient to upper bound the regret of this modified algorithm, which draws randomness only once.

Thus it is sufficient to upper bound the regret of this modified algorithm, which draws randomness only once.

\section{Omitted Proofs from Section~\ref{sec:cont-lin}}
\label{app:complete-info}

\subsection{Bounding the Laplacian Error} 
The upper bound on the $\error$ term is identical in all settings, since it only depends on the input noise distribution, which is the same for all variants and for which it does not matter whether $X$ is the set of contexts that will arrive or a separator. In subsequent sections we will upper bound the stability of the algorithm in each setting.

\begin{lemma}[Laplacian Error Bound]\label{lem:laplace-error-bound-lin}
Let $\{z\}$ denote a sample from the random sequence of fake samples used by $\contlin(X,\epsilon)$. Then:
\begin{equation}
\textsc{Error}=\E_{\{z\}}\left[\max_{\pi \in \Pi} \sum_{x\in X} \dotp{\pi(x)}{\ell_x}\right]-\E_{\{z\}}\left[\min_{\pi \in \Pi} \sum_{x\in X} \dotp{\pi(x)}{\ell_x}\right]\leq \frac{10}{\epsilon}\sqrt{dm}\log(N)
\end{equation}
\end{lemma}
\begin{proof}
First we start by observing that each random variable $\ell_x(j)$ is distributed i.i.d. according to a Laplace$(\epsilon)$ distribution. Since a Laplace distribution is symmetric around $0$, we get that $\ell_x(j)$ and $-\ell_x(j)$ are distributed identically. Thus we can write:
\begin{equation*}
\E_{\{z\}}\left[\min_{\pi \in \Pi} \sum_{x\in X} \dotp{\pi(x)}{\ell_x}\right] = \E_{\{z\}}\left[\min_{\pi \in \Pi} \sum_{x\in X}\dotp{\pi(x)}{-\ell_x}\right] = -\E_{\{z\}}\left[\max_{\pi \in \Pi} \sum_{x\in X} \dotp{\pi(x)}{\ell_x}\right]
\end{equation*}
Hence we get:
\begin{equation}
\textsc{Error} = 2\cdot\E_{\{z\}}\left[\max_{\pi \in \Pi} \sum_{x\in X} \dotp{\pi(x)}{\ell_x}\right]
\end{equation}
We now bound the latter expectation via a moment generating function approach. For any $\lambda \geq 0$: 
\begin{align*}
\E_{\{z\}}\left[\max_{\pi \in \Pi} \sum_{x\in X} \dotp{\pi(x)}{\ell_x}\right] =~&  \frac{1}{\lambda}\E_{\{z\}}\left[\max_{\pi \in \Pi} \lambda\sum_{x\in X} \dotp{\pi(x)}{\ell_x}\right]\\
=~& \frac{1}{\lambda} \log\left\{\exp\left\{\E_{\{z\}}\left[\max_{\pi \in \Pi} \lambda\sum_{x\in X} \dotp{\pi(x)}{\ell_x}\right]\right\}\right\}
\end{align*} 
By convexity and monotonicity of the exponential function:
\begin{align*}
\E_{\{z\}}\left[\max_{\pi \in \Pi} \sum_{x\in X}\dotp{\pi(x)}{\ell_x}\right] \leq~& \frac{1}{\lambda} \log\left\{\E_{\{z\}}\left[\max_{\pi \in \Pi} \exp\left\{\lambda\sum_{x\in X}\dotp{\pi(x)}{\ell_x}\right\}\right]\right\}\\
\leq~& \frac{1}{\lambda} \log\left\{\sum_{\pi \in \Pi}\E_{\{z\}}\left[ \exp\left\{\lambda\sum_{x\in X} \dotp{\pi(x)}{\ell_x}\right\}\right]\right\}\\
\leq~& \frac{1}{\lambda} \log\left\{\sum_{\pi \in \Pi}\prod_{x\in X}\E\left[ \exp\left\{\lambda\dotp{\pi(x)}{\ell_x}\right\}\right]\right\}\\
=~&\frac{1}{\lambda} \log\left\{\sum_{\pi \in \Pi}\prod_{x\in X}\E\left[ \exp\left\{\lambda\sum_{j: \pi(x)(j)=1}\ell_x(j)\right\}\right]\right\}\\
=~&\frac{1}{\lambda} \log\left\{\sum_{\pi \in \Pi}\prod_{x\in X}\prod_{j:\pi(x)(j)=1}\E\left[ \exp\left\{\lambda\ell_x(j)\right\}\right]\right\}\\
\end{align*} 
For any $j\in [K]$ and $x\in X$, $\ell_x(j)$ is a Laplace$(\epsilon)$ random variable. Hence, the quantity $\E\left[\exp\{\lambda \ell_x(j)\}\right]$ is the moment generating function of the Laplacian distribution evaluated at $\lambda$, which is equal to $\frac{1}{1-\frac{\lambda^2}{\epsilon^2}}$ provided that $\lambda < \epsilon$. Since $\sup_{x,\pi}|\{j\in [K]: \pi(x)(j)\}|\leq m$, we get:
\begin{align*}
\E_{\{z\}}\left[\max_{\pi \in \Pi} \sum_{x\in X}\dotp{\pi(x)}{\ell_x}\right]\leq~&
 \frac{1}{\lambda} \log\left\{N \left(\frac{1}{1-\frac{\lambda^2}{\epsilon^2}}\right)^{dm}\right\}
 ~=~ \frac{1}{\lambda}\log(N) + \frac{dm}{\lambda} \log\left(\frac{1}{1-\frac{\lambda^2}{\epsilon^2}}\right)
 \end{align*}
By simple calculus, it is easy to derive that $\frac{1}{1-x}\leq e^{2x}$ for any $x\leq \frac{1}{4}$.\footnote{
Consider the function $f(x)=(1-x)e^{2x} - 1$. Then $f(0)=0$ and $f'(x) = e^{2x}(1-2x)$, which is $\geq 0$ for $0\leq x\leq 1/2$.} Thus as long as we pick $\lambda \leq \frac{\epsilon}{2}$, we get:
\begin{align*}
\E_{\{z\}}\left[\max_{\pi \in \Pi} \sum_{x\in X} \dotp{\pi(x)}{\ell_x}\right]\leq~& 
\frac{1}{\lambda}\log(N) + \frac{dm}{\lambda} \log\left(\exp\left\{\frac{\lambda^2}{\epsilon^2}\right\}\right)~=~\frac{1}{\lambda}\log(N) + \frac{2dm\lambda}{\epsilon^2} 
 \end{align*}
 Picking $\lambda = \frac{\epsilon}{2\sqrt{dm}}$ and since $N\geq 2$:
\begin{align*}
\E_{\{z\}}\left[\max_{\pi \in \Pi} \sum_{x\in X} \dotp{\pi(x)}{\ell_x}\right]\leq~& \frac{2\sqrt{dm}\log(N)}{\epsilon}+\frac{\sqrt{2dm}}{\epsilon}\leq \frac{5\sqrt{dm}\log(N)}{\epsilon}
 \end{align*}
\end{proof}

\subsection{Bounding Stability: Transductive Setting}

We now turn to bounding the stability in the transductive combinatorial optimization setting.
Combining the following lemma with the error bound in Lemma~\ref{lem:laplace-error-bound-lin} and applying Theorem~\ref{thm:general-regret} proves the first claim of Theorem~\ref{thm:main-complete}.
\begin{lemma}[Transductive Stability] \label{lem:non-l-stability}
For all $t \in [T]$ and for any sequence $\param^{1:t}$ of contexts $x^{1:t}$ and loss functions $f^{1:t}$ with $f^i:\{0,1\}^K\rightarrow \R^K$, the stability of $\contlin(X,\epsilon)$ is upper bounded by:
\begin{align*}
\E_{\{z\}}\left[f^t(\pi^t(x^t))-f^t(\pi^{t+1}(x^t))\right] \le 4\epsilon K\cdot \|f^t\|_{*}^2
\end{align*}
\end{lemma}
\begin{proof}
By the definition of $\|f^t\|_{*}$:
\begin{align*}
\E_{\{z\}}\left[f^t(\pi^t(x^t))-f^t(\pi^{t+1}(x^t))\right] \le~& 2\|f^t\|_{*}\Pr\left[\pi^t(x^t) \neq \pi^{t+1}(x^t)\right]
\end{align*}
Now observe that:
\begin{align*}
\Pr\left[\pi^t(x^t) \neq \pi^{t+1}(x^t)\right]\leq~& \sum_{j\in K} \left(\Pr[j\in \pi^t(x^t), j\notin \pi^{t+1}(x^t)]+\Pr[j\notin \pi^t(x^t), j\in \pi^t(x^t)]\right)
\end{align*}

We bound  the probability $\Pr[j\in \pi^t(x^t), j\notin \pi^{t+1}(x^t)]$. We condition on all random variables of $\{z\}$ except for the random variable $\ell_{x^t}(j)$, i.e. the random loss placed at coordinate $j$ on the sample associated with context $x^t$. Denote the event corresponding to an assignment of all these other random variables as ${\cal E}_{-x^tj}$. Let $\ell_{x^tj}$ denote a loss vector which is $\ell_{x^t}(j)$ on the $j$-th coordinate and zero otherwise. Also let:
\begin{equation}
\Phi(\pi) = \sum_{\tau=1}^{t-1} f^\tau(\pi(x^\tau)) + \sum_{x\in X-\{x^t\}} \dotp{\pi(x)}{\ell_x}+\dotp{\pi(x^t)}{\ell_{x^t}-\ell_{x^tj}}
\end{equation}
Let $\pi^* = \argmin_{\pi\in \Pi:j\in \pi(x^t)} \Phi(\pi)$ and $\tilde{\pi} = \min_{\pi\in \Pi: j\notin \pi(x^t)}  \Phi(\pi)$. The event that $\{j\in \pi^t(x^t)\}$ happens only if:
\begin{equation}
\Phi(\pi^*)+\ell_{x^t}(j) \leq \Phi(\tilde{\pi})
\end{equation}
Let  and $\nu = \Phi(\tilde{\pi})-\Phi(\pi^*)$. Thus $j\in\pi^t(x^t)$ only if:
\begin{equation}
\ell_{x^t}(j) \leq \nu
\end{equation}
Now if:
\begin{equation}
\ell_{x^t}(j) < \nu-2\|f^t\|_{*}
\end{equation}
then it is easy to see that $\{j\in \pi^{t+1}(x^t)\}$, since an extra loss of $f^t(a)\in [0,1]$ cannot push $j$ out of the optimal solution. More elaborately, for any other policy $\pi\in \Pi$, such that $j\notin \pi(x^t)$, the loss of $\pi^*$ including time-step $t$ is bounded as:
\begin{align*}
\Phi(\pi^*)+ \ell_{x^t}(j)+f^t(\pi^*(x^t))<~& \Phi(\pi)-2\|f^t\|_*+f^t(\pi^*(x^t))\\
<~& \Phi(\pi)-\|f^t\|_*\\
<~& \Phi(\pi) +f^t(\pi(x^t))
\end{align*}
Thus any policy $\pi$, such that $j\notin \pi(x^t)$ is suboptimal after seeing the loss at time-step $t$. Thus
\begin{align*}
\Pr[j\in \pi^t(x^t), j\notin \pi^{t+1}(x^t)~|~{\cal E}_{-x^tj}]\leq \Pr[\ell_{x^t}(j) \in [\nu-2\|f^t\|_*,\nu]~|~{\cal E}_{-x^tj}]
\end{align*}
Since all other random variables are independent of $\ell_{x^t}(j)$ and $\ell_{x^t}(j)$ is a Laplacian with 
parameter $\epsilon$:
\begin{align*}
\Pr[\ell_{x^t}(j) \in [\nu-2\|f^t\|_*,\nu]~|~{\cal E}_{-x^tj}]=~&\Pr[\ell_{x^t}(j) \in [\nu-2\|f^t\|_*,\nu]]\\ =~&\frac{\epsilon}{2}\int_{\nu-2\|f^t\|_*}^{\nu}e^{-\epsilon|z|}dz\leq \frac{\epsilon}{2}\int_{\nu-2\|f^t\|_*}^{\nu}dz \leq \epsilon\|f^t\|_*
\end{align*}
Similarly it follows that that: $\Pr[j\notin \pi^t(x^t) \text{ and } j\in \pi^{t+1}(x^t)]\leq\epsilon\|f^t\|_*$. 
To sum we get that:
\begin{align*}
\E_{\{z\}}\left[f^t(\pi^t(x^t))-f^t(\pi^{t+1}(x^t))\right] \le~& 2\|f^t\|_*\Pr\left[\pi^t(x^t) \neq \pi^{t+1}(x^t)\right]\leq 4\epsilon K \|f^t\|_*^2
\end{align*}
\end{proof}

\subsection{Bounding Stability: Transductive Setting with Linear Losses}

In the transductive setting with linear losses, we provide a significantly more refined stability bound, which enables applications to partial information or bandit settings. 
As before, combining this stability bound with the error bound in Lemma~\ref{lem:laplace-error-bound-lin} and applying Theorem~\ref{thm:general-regret} gives the second claim of Theorem~\ref{thm:main-complete}.

\begin{lemma}[Multiplicative Stability]
\label{lem:multiplicative_stability-lin}
For any sequence $\param^{1:T}$ for all $t \in [T]$ of contexts and non-negative linear loss functions, the stability of $\contlin(X,\epsilon)$ in the transductive setting, is upper bounded by:
\begin{align*}
\E_{\{z\}}\left[\dotp{\pi^t(x^t)}{\ell^t}-\dotp{\pi^{t+1}(x^t)}{\ell^t}\right] \le \epsilon\cdot \E\left[\dotp{\pi^t(x^t)}{\ell^t}^2\right]
\end{align*}
\end{lemma}
\begin{proof}
To prove the result we first must introduce some additional terminology. 
For a sequence of parameters $\param^{1:t}$, let $\phi^t \in \R^{dK}$ be a vector with $\phi_{x,j}^t = \sum_{\tau \leq t: x^\tau = x} \ell^\tau(j)$.
The component of this vector corresponding to context $x \in X$ and coordinate $j \in [K]$ is the cumulative loss associated with that coordinate on the subset of time points when context $x$ appeared. 
Note that this vector $\phi^t$ is a sufficient statistic, since for any fixed policy $\pi$:
\begin{align}
\sum_{\tau=1}^t \ell(\pi,\param^\tau) = \sum_{x\in X}\sum_{\tau\leq t: x^\tau=x} \dotp{\pi(x)}{\ell^\tau} = \sum_{x\in X}\dotp{\pi(x)}{\phi_{x}^t}
\end{align}
where $\phi_{x}^t = \sum_{\tau\leq t: x^\tau=x} \ell^\tau$. 

We denote with $z\in \R^{dK}$ the sufficient statistic that corresponds to the fake sample sequence $\{z\}$ and with $\phi^t$ the sufficient statistics for the parameter sequence $\param^{1:t}$.
Observe that the sufficient statistic for the augmented sequence $\{z\}\cup\param^{1:t}$ is simply $z+\phi^t$.
For any sequence of parameters $\param^{1:T}$ we will be denoting with $\phi^{1:T}$ the sequence of $d\cdot K$ dimensional cumulative loss vectors.
We will also overload notation and denote with $M(\phi^t)=M(\param^{1:t})$ the best policy on a sequence $\param^{1:t}$ with statistics $\phi^t$.

Consider a specific sequence $\param^{1:T}$ and a specific time step $t$.
Define, for each $\pi \in \Pi$, a sparse tuple $\param_\pi^t = (x^t, \ell^t_\pi)$ where $\ell^t_\pi(j) =  \ell^t(j)$ if $\pi(x^t)(j) = 1$ and zero otherwise, i.e. we zero out coordinates of the true loss vector that were not picked by the policy $\pi$. 
Moreover, define with $\phi_\pi^t$ the sufficient statistic of the sequence $\phi(\param^{1:t-1} \cup \param_\pi^t)$ for each $\pi$. 
We define $1+|\Pi|$ distributions over $|\Pi|$, via their probability density functions, as follows:
\begin{align*}
p^{t}(\pi) ~=~& \Pr[ M(z + \phi^{t-1})=\pi]\\
\forall \pi^*\in \Pi: p^{t+1}_{\pi^*}(\pi) ~=~& \Pr\left[M(z + \phi_{\pi^*}^t)=\pi\right]
\end{align*}
At the end of this proof, we will show that $p^{t+1}_\pi(\pi) \le p^{t+1}(\pi)$. 
Moreover, we denote for convenience:
\begin{align*}
\FTPL^t =~&  \E_z[\dotp{\pi^t(x^t)}{\ell^t}]~=~\E_{\pi \sim p^{t}}\left[ \dotp{\pi(x^t)}{\ell^t}\right] \\
\BTPL^t =~& \E_z[\dotp{\pi^{t+1}(x^t)}{\ell^t}]~=~\E_{\pi \sim p^{t+1}}\left[ \dotp{\pi(x^t)}{\ell^t}\right]
\end{align*}
We will construct a mapping $\mu_\pi: \R^{dK} \rightarrow \R^{dK}$ such that for any $z \in \R^{dK}$,
\begin{align*}
M(z + \phi_{\pi}^{t}) = M(\mu_\pi(z) + \phi^{t-1})
\end{align*}
Notice that $\mu_\pi(z) = z + \phi_\pi^t - \phi^{t-1}$.
Now,
\begin{align*}
p^{t}(\pi) &= \int_z \mathbf{1}[\pi = M(z + \phi^{t-1})] f(z)dz\\
& = \int_z \mathbf{1}[\pi = M(\mu_\pi(z) + \phi^{t-1})] f(\mu_\pi(z))dz\\
& = \int_z \mathbf{1}[\pi = M(z + \phi_\pi^t)] f(\mu_\pi(z))dz
\end{align*}
Now observe that for any $z\in \R^{dK}$:
\begin{align*}
f(\mu_\pi(z)) =~& \exp\{- \epsilon\left( \|z+\phi_\pi^{t} - \phi^{t-1}\|_1-\|z\|_1\right)\}f(z)\\
\leq~&\exp\{- \epsilon\left( \|z+\phi_\pi^{t} - \phi^{t-1}\|_1-\|z+\phi_{\pi}^t-\phi^{t-1}\|_1-\|\phi^{t-1}-\phi_{\pi}^t\|_1\right)\}f(z)\\
\leq~&  \exp\{\epsilon\|\phi_\pi^{t} - \phi^{t-1}\|_1\}f(z)\\
=~&\exp\{\epsilon \dotp{\pi(x^t)}{\ell^t}\} f(z)
\end{align*}
Substituting in this bound, we have,
\begin{align*}
p^{t}(\pi) \le \exp\{\epsilon \dotp{\pi(x^t)}{\ell^t}\}\cdot p_{\pi}^{t+1}(\pi) \leq \exp\{\epsilon \dotp{\pi(x^t)}{\ell^t}\}\cdot p^{t+1}(\pi)
\end{align*}
Re-arranging and lower bounding $\exp\{-x\} \geq (1-x)$:
\begin{equation}
p^{t+1}(\pi) \ge \exp\{-\epsilon \dotp{\pi(x^t)}{\ell^t}\}\cdot p^{t}(\pi) \ge (1-\epsilon \dotp{\pi(x^t)}{\ell^t})\cdot p^{t}(\pi)
\end{equation}
Using the definition of $\FTPL^t$ and $\BTPL^t$, this gives,
\begin{align*}
\BTPL^t &= \sum_\pi p^{t+1}(\pi)\dotp{\pi(x^t)}{\ell^t}
\ge \sum_\pi (1-\epsilon \dotp{\pi(x^t)}{\ell^t}) p^{t}(\pi)\dotp{\pi(x^t)}{\ell^t}\\
&= \FTPL^t - \epsilon \sum_\pi  p^{t}(\pi)\dotp{\pi(x^t)}{\ell^t}^2\\
&=\FTPL^t - \epsilon \E\left[\dotp{\pi(x^t)}{\ell^t}^2\right]
\end{align*}

We will finish the proof by showing that $p_\pi^{t+1}(\pi) \le p^{t+1}(\pi)$ for all $\pi \in \Pi$.
For succinctness we drop the dependence on $t$. 
Notice that for any other policy $\pi' \ne \pi$
\begin{align*}
\Ell(\pi, z+\phi_\pi^t) \leq \Ell(\pi', z+\phi_{\pi}^t) \Rightarrow \Ell(\pi, z+\phi^{t}) \leq \Ell(\pi', z+\phi^t).
\end{align*}
And similarly for strict inequalities. This follows since the loss of $\pi$ remains unchanged, but the loss of $\pi'$ can only go up, since $\ell^t_\pi(j) \le \ell^t(j)$ (as losses are non-negative). For simplicity assume that $\pi$ always wins in case of ties, though the argument goes through if we assume a deterministic tie-breaking rule based on some global ordering of policies.  
Thus,
\begin{align*}
p^{t+1}(\pi) = \PP\left[\bigcap_{\pi'} \Ell(\pi, z+\phi^t) \leq \Ell(\pi', z+\phi^t)\right] \le \PP\left[\bigcap_{\pi'} \Ell(\pi, z+\phi^t_\pi) \leq \Ell(\pi', z+\phi^t_\pi)\right] = p^{t+1}_\pi(\pi)
\end{align*}
as claimed. 
\end{proof}

\subsection{Bounding Stability: Small Separator Setting}

Finally, we prove the third claim in Theorem~\ref{thm:main-complete}.
This involves a new stability bound for the small separator setting.

\begin{lemma}[Stability for small separator] \label{lem:small-set-stability}
For any $t \in [T]$ and any sequence $\param^{1:t}$ of contexts $x^{1:t}$ and losses $f^{1:t}$ with $f^i:\{0,1\}^K\rightarrow \R^K$, the stability of $\contlin(\epsilon)$, when $X$ is a separator, is upper bounded by:
\begin{align*}
\E_{\{z\}}\left[f^t(\pi^t(x^t))-f^t(\pi^{t+1}(x^t))\right] \le 4\epsilon K d\cdot \|f^t\|_{*}^2
\end{align*}
\end{lemma}
\begin{proof}
By the definition of $\|f^t\|_{*}$:
\begin{align*}
\E_{\{z\}}\left[f^t(\pi^t(x^t))-f^t(\pi^{t+1}(x^t))\right] \le~& 2\|f^t\|_{*}\Pr\left[\pi^t(x^t) \neq \pi^{t+1}(x^t)\right]\leq 2\|f^t\|_{*}\Pr[\pi^t\neq \pi^{t+1}]
\end{align*}
Since $X$ is a separator, $\pi^t\neq \pi^{t+1}$ if and only if there exists a context $x\in X$, such that $\pi^t(x)\neq \pi^{t+1}(x)$. Otherwise the two policies are identical. Thus we have by two applications of the union bound:
\begin{align*}
\Pr[\pi^t\neq \pi^{t+1}]\leq~& \sum_{x\in X} \Pr[\pi^t(x)\neq \pi^{t+1}(x)]\\
\leq~& \sum_{x\in X} \sum_{j\in K}\left( \Pr[j\in \pi^t(x), j\notin \pi^{t+1}(x)] + \Pr[j\notin \pi^t(x), j\in \pi^{t+1}(x)]\right)
\end{align*}

We bound  the probability $\Pr[j\in \pi^t(x), j\notin \pi^{t+1}(x)]$. We condition on all random variables of $\{z\}$ except for the random variable $\ell_{x}(j)$, i.e. the random loss placed at coordinate $j$ on the sample associated with context $x$. Denote the event corresponding to an assignment of all these other random variables as ${\cal E}_{-xj}$. Let $\ell_{xj}$ denote a loss vector which is $\ell_{x}(j)$ on the $j$-th coordinate and zero otherwise. Also let:
\begin{equation}
\Phi(\pi) = \sum_{\tau=1}^{t-1} f^\tau(\pi(x^\tau)) + \sum_{x'\neq x} \dotp{\pi(x')}{\ell_{x'}}+\dotp{\pi(x)}{\ell_{x}-\ell_{xj}}
\end{equation}
Let $\pi^* = \argmin_{\pi\in \Pi:j\in \pi(x)} \Phi(\pi)$ and $\tilde{\pi} = \min_{\pi\in \Pi: j\notin \pi(x)}  \Phi(\pi)$. The event that $\{j\in \pi^t(x)\}$ happens only if:
\begin{equation}
\Phi(\pi^*)+\ell_{x}(j) \leq \Phi(\tilde{\pi})
\end{equation}
Let  and $\nu = \Phi(\tilde{\pi})-\Phi(\pi^*)$. Thus $j\in\pi^t(x)$ only if:
\begin{equation}
\ell_{x}(j) \leq \nu
\end{equation}
Now if:
\begin{equation}
\ell_{x}(j) < \nu-2\|f^t\|_{*}
\end{equation}
then it is easy to see that $\{j\in \pi^{t+1}(x)\}$, since an extra loss of $f^t(a)\leq \|f^t\|_{*}$ cannot push $j$ out of the optimal solution. More elaborately, for any other policy $\pi\in \Pi$, such that $j\notin \pi(x)$, the loss of $\pi^*$ including time-step $t$ is bounded as:
\begin{align*}
\Phi(\pi^*)+ \ell_{x}(j)+f^t(\pi^*(x^t))<~& \Phi(\pi)-2\|f^t\|_{*}+f^t(\pi^*(x^t))\\
<~& \Phi(\pi)-\|f^t\|_*\\
<~& \Phi(\pi) +f^t(\pi(x^t))
\end{align*}
Thus any policy $\pi$, such that $j\notin \pi(x)$ is suboptimal after seeing the loss at time-step $t$. Thus
\begin{align*}
\Pr[j\in \pi^t(x), j\notin \pi^{t+1}(x)~|~{\cal E}_{-xj}]\leq \Pr[\ell_{x}(j) \in [\nu-2\|f^t\|_{*},\nu]~|~{\cal E}_{-xj}]
\end{align*}
Since all other random variables are independent of $\ell_{x}(j)$ and $\ell_{x}(j)$ is a Laplacian with 
parameter $\epsilon$:
\begin{align*}
\Pr[\ell_{x}(j) \in [\nu-2\|f^t\|_{*},\nu]~|~{\cal E}_{-xj}]=~&\Pr[\ell_{x}(j) \in [\nu-2\|f^t\|_{*},\nu]]\\ =~&\frac{\epsilon}{2}\int_{\nu-2\|f^t\|_{*}}^{\nu}e^{-\epsilon|z|}dz\leq \frac{\epsilon}{2}\int_{\nu-2\|f^t\|_{*}}^{\nu}dz \leq \epsilon\|f^t\|_{*}
\end{align*}
Similarly it follows that that: $\Pr[j\notin \pi^t(x), j\in \pi^{t+1}(x)]\leq \epsilon\|f^t\|_{*}$. 
To sum we get that:
\begin{align*}
\E_{\{z\}}\left[f^t(\pi^t(x^t))-f^t(\pi^{t+1}(x^t))\right] \le~& 2\|f^t\|_{*}\Pr\left[\pi^t \neq \pi^{t+1}\right]\leq 4\epsilon Kd\cdot \|f^t\|_{*}^2
\end{align*}
\end{proof}

\section{Omitted Proofs from Section~\ref{sec:contextual-bandits}}

\subsection{Proof of Theorem \ref{thm:semi-bandit-regret}: Transductive Setting}

Consider the expected loss of the bandit algorithm at time-step $t$, conditional on $\F^{t-1}$:
\begin{equation}
\E[\dotp{\pi^t(x^t)}{\ell^t}~|~\F^{t-1}]=\sum_{j=1}^{K} q^t(j)\cdot \ell^t(j) \leq \sum_{j=1}^{K} q^t(j)\cdot\E\left[\hat{\ell}^t(j)~|~\F^{t-1}\right] + \sum_{j=1}^{K} \ell^t(j) q^t(j)\cdot(1-q^t(j))^L
\end{equation}
As was observed by \cite{neu2013efficient}, the second quantity can be upper bounded by $\frac{K}{e L}\|\ell^t\|_{*}$, since $q(1-q)^L\leq qe^{-Lq}\leq \frac{1}{eL}$.

Now observe that: $\sum_{j\in K} q^t(j)\cdot \E\left[\hat{\ell}^t(j)~|~\F^{t-1}\right]$ is the expected loss of the full feedback algorithm on the sequence of losses it observed and conditional on the history of play. By the regret bound of $\contlin(X,\epsilon)$, given in case $2$ of Theorem \ref{thm:main-complete}, we have that for any policy $\pi^*$:
\begin{align*}
\E\left[\sum_{t=1}^T \sum_{j=1}^{K} q^t(j)\cdot \hat{\ell}^t(j)\right] \leq~& \E\left[\sum_{t=1}^T \dotp{\pi^*(x^t)}{\hat{\ell}^t}\right] +  \epsilon \E\left[\sum_{t=1}^T\sum_{\pi\in \Pi}p^t(\pi)\dotp{\pi(x^t)}{\hat{\ell}^t}^2\right]+\frac{10}{\epsilon}\sqrt{dm}\log(N)
\end{align*}
Using the fact that expected estimates $\hat{\ell}$ are upper bounded by true losses:
\begin{align*}
\E\left[\sum_{t=1}^T \sum_{j=1}^{K} q^t(j)\hat{\ell}^t(j)\right] \leq~&\min_{\pi^*\in \Pi}\E\left[\sum_{t=1}^T  \dotp{\pi^*(x^t)}{\hat{\ell}^t}\right] + \epsilon \E\left[\sum_{t=1}^T\sum_{\pi\in \Pi}p^t(\pi)\dotp{\pi(x^t)}{\hat{\ell}^t}^2\right]+\frac{10}{\epsilon}\sqrt{dm}\log(N)
\end{align*}
Combining the two upper bounds, we get that the expected regret of the bandit algorithm is upper bounded by:
\begin{align*}
\regret \leq~& \epsilon \E\left[\sum_{t=1}^T\sum_{\pi\in \Pi}p^t(\pi)\dotp{\pi(x^t)}{\hat{\ell}^t}^2\right]+\frac{10}{\epsilon}\sqrt{dm}\log(N)+\frac{K}{eL}\sum_{t=1}^T\E\left[\|\ell^t\|_{*}\right]
\end{align*}
Now observe that, by a simple norm inequality and re-grouping:
\begin{align*}
\sum_{\pi\in \Pi} p^t(\pi)\dotp{\pi(x^t)}{\hat{\ell}^t}^2 = \sum_{\pi\in \Pi}p^t(\pi) \left(\sum_{j\in \pi(x^t)}\hat{\ell}^t(j)\right)^2\leq m\sum_{\pi\in \Pi}p^t(\pi) \sum_{j\in \pi(x^t)}\hat{\ell}^t(j)^2 = m\sum_{j\in [K]}q^t(j) \hat{\ell}^t(j)^2
\end{align*}
Thus we get:
\begin{align*}
\regret \leq~& \epsilon m\sum_{t=1}^T\E\left[\sum_{j\in [K]}q^t(j)\hat{\ell}^t(j)^2\right]+\frac{10}{\epsilon}\sqrt{dm}\log(N)+\frac{K}{eL}\sum_{t=1}^T\E\left[\|\ell^t\|_{*}\right]
\end{align*}
Now we bound each of the terms in the first summation, conditional on any history of play:
\begin{align*}
\sum_{j\in [K]}q^t(j)\E\left[\hat{\ell}^t(j)^2~|~\F^{t-1}\right] =~&\sum_{j\in [K]} q^t(j) q^t(j)\ell^t(j)^2\E\left[J^t(j)^2~|~\F^{t-1},j\in \pi^t(x^t)\right]
\end{align*}
Each $J^t(j)$ conditional on $\F^{t-1}$ and $j\in \pi^t(x^t)$ is distributed according to a geometric distribution with mean $q^t(j)$ truncated at $L$. Hence, it is stochastically dominated by a geometric distribution with mean $q^t(j)$. By known properties, if $X$ is a geometrically distributed random variable with mean $q$, then $\E[X^2]=Var(X)+(\E[X])^2 = \frac{1-q}{q^2}+\frac{1}{q^2} =\frac{2-q}{q^2}\leq \frac{2}{q^2}$. Thus we have:
\begin{align*}
\sum_{j\in [K]}q^t(j)\E\left[\hat{\ell}^t(j)^2~|~\F^{t-1}\right] \leq~&
\sum_{j\in [K]} q^t(j)^2 \ell^t(j)^2\frac{2}{q^t(j)^2}=2\sum_{j=1}^{K} \ell^t(j)^2 \leq 2K\|\ell^t\|_\infty^2
\end{align*}
Combining all the above we get the theorem.

\subsection{Proof of Theorem \ref{thm:semi-bandit-regret}: Small Separator Setting}

Consider the expected loss of the bandit algorithm at time-step $t$, conditional on $\F^{t-1}$:
\begin{equation}
\E[\dotp{\pi^t(x^t)}{\ell^t}~|~\F^{t-1}]=\sum_{j=1}^{K} q^t(j)\cdot \ell^t(j) \leq \sum_{j=1}^{K} q^t(j)\cdot\E\left[\hat{\ell}^t(j)~|~\F^{t-1}\right] + \sum_{j=1}^{K} \ell^t(j) q^t(j)\cdot(1-q^t(j))^L
\end{equation}
As was observed by \cite{neu2013efficient}, the second quantity can be upper bounded by $\frac{K}{e L}\|\ell^t\|_{*}$, since $q(1-q)^L\leq qe^{-Lq}\leq \frac{1}{eL}$.

Now observe that: $\sum_{j\in K} q^t(j)\cdot \E\left[\hat{\ell}^t(j)~|~\F^{t-1}\right]$ is the expected loss of the full feedback algorithm on the sequence of losses it observed and conditional on the history of play. By the regret bound of $\contlin(X,\epsilon)$, given in case $3$ of Theorem \ref{thm:main-complete}, we have that for any policy $\pi^*$:
\begin{align*}
\E\left[\sum_{t=1}^T \sum_{j=1}^{K} q^t(j)\cdot \hat{\ell}^t(j)\right] \leq~& \E\left[\sum_{t=1}^T \dotp{\pi^*(x^t)}{\hat{\ell}^t}\right] +  4\epsilon K d \cdot \sum_{t=1}^T \E\left[\|\hat{f}^t\|_{*}^2\right]+\frac{10}{\epsilon}\sqrt{dm}\log(N)\\
\leq~& \sum_{t=1}^T \dotp{\pi^*(x^t)}{\hat{\ell}^t} +  4\epsilon K d \cdot \sum_{t=1}^T \E\left[\|\hat{\ell}^t\|_{1}^2\right]+\frac{10}{\epsilon}\sqrt{dm}\log(N)
\end{align*}
Using the fact that expected estimates $\hat{\ell}$ are upper bounded by true losses:
\begin{align*}
\E\left[\sum_{t=1}^T \sum_{j=1}^{K} q^t(j)\hat{\ell}^t(j)\right] \leq~&\min_{\pi^*\in \Pi}\E\left[\sum_{t=1}^T  \dotp{\pi^*(x^t)}{\hat{\ell}^t}\right] + 4\epsilon K d \cdot \sum_{t=1}^T \E\left[\|\hat{\ell}^t\|_{1}^2\right]+\frac{10}{\epsilon}\sqrt{dm}\log(N)
\end{align*}
Combining the two upper bounds, we get that the expected regret of the semi-bandit algorithm is upper bounded by:
\begin{align*}
\regret \leq~& 4\epsilon K d \cdot \sum_{t=1}^T \E\left[\|\hat{\ell}^t\|_{1}^2\right]+\frac{10}{\epsilon}\sqrt{dm}\log(N)+\frac{K}{eL}\sum_{t=1}^T\E\left[\|\ell^t\|_{*}\right]
\end{align*}
Now we bound each of the terms in the first summation, conditional on any history of play:
\begin{align*}
\E\left[\|\hat{\ell}^t\|_{1}^2~|~\F^{t-1}\right]\leq m\E\left[\|\hat{\ell}^t\|_{2}^2\right] =m\sum_{j\in [K]}\E\left[\hat{\ell}^t(j)^2\right] =m\sum_{j\in [K]}  q^t(j)\ell^t(j)^2\E\left[J^t(j)^2~|~\F^{t-1},j\in \pi^t(x^t)\right]
\end{align*}
Each $J^t(j)$ conditional on $\F^{t-1}$ and $j\in \pi^t(x^t)$ is distributed according to a geometric distribution with mean $q^t(j)$ truncated at $L$. Hence, it is stochastically dominated by a geometric distribution with mean $q^t(j)$. By known properties, if $X$ is a geometrically distributed random variable with mean $q$, then $\E[X^2]=Var(X)+(\E[X])^2 = \frac{1-q}{q^2}+\frac{1}{q^2} =\frac{2-q}{q^2}\leq \frac{2}{q^2}$. Moreover, trivially $\E[X^2]\leq L^2$, since $X$ is truncated at $L$. Thus we have:
\begin{align*}
\E\left[\|\hat{\ell}^t\|_{1}^2~|~\F^{t-1}\right]\leq
m\sum_{j\in[K]} q^t(j) \ell^t(j)^2\min\left\{\frac{2}{q^t(j)^2},L^2\right\}\leq
m\|\ell^t\|_*^2\sum_{j\in[K]} \min\left\{\frac{2}{q^t(j)},q^t(j) L^2\right\}
\end{align*}
Now observe that: $\min\left\{\frac{2}{q^t(j)},q^t(j) L^2\right\}\leq 2L$, since either $\frac{1}{q^t(j)}\leq L$ or otherwise, $q^t(j)L^2 \leq \frac{1}{L}L\leq L$. Thus we get:
\begin{align*}
\E\left[\|\hat{\ell}^t\|_{1}^2~|~\F^{t-1}\right]\leq 2LKm\|\ell^t\|_*^2
\end{align*}
Combining all the above we get the theorem.

\section{Omitted Proofs from Section~\ref{sec:path-length}}
\subsection{Proof of Theorem~\ref{thm:oftpl-bound}}
Similar to the analysis in Section \ref{sec:cont-lin}, the proof of the Theorem is broken apart in two main Lemmas. The first lemma is an analogue of Theorem \ref{thm:general-regret} for algorithms that use a predictor. This lemma can be phrased in the general online learning setting analyzed in Section \ref{sec:oracles}. The second Lemma is an anaolgue of our multiplicative stability Lemma \ref{lem:multiplicative_stability-lin}.

Let
\begin{equation}
\rho^t = M(\{z\}\cup \param^{1:t})
\end{equation}
denote the policy that would have been played at time-step $t$ if the predictor was equal to the actual loss vector that occured at time-step $t$. Moreover, for succinctness we will denote with $a^t=\pi^t(x^t)$ and with $b^t=\rho^t(x^t)$. 

\begin{lemma}[Follow vs Be the Leader with Predictors] The regret of a player under the optimistic FTPL and with respect to any $\pi^*\in \Pi$ is upper bounded by:
\begin{equation}
\regret \leq \sum_{t=1}^T \E\left[\Delta Q^t(a^t)-\Delta Q^t(b^t)\right] +
\E[\error]
\end{equation}
where  $\Delta Q^t(a) = f^t(a)-Q^t(a)$ and $\error =   \max_{\pi\in \Pi}\sum_{x\in X}\dotp{\pi(x)}{\ell_x}-\min_{\pi\in \Pi}\sum_{x\in X}\dotp{\pi(x)}{\ell_x}$.
\end{lemma} 
\begin{proof}
Consider the augmented sequence $(x^1,\pred^1),(x^1,f^1-\pred^1),(x^2,\pred^2),(x^2,f^2-\pred^2),\ldots$, where each observation $(x^t,f^t)$ is replaced by two observations $(x^t,\pred^t)$ followed by $(x^t,f^t-\pred^t)$. Observe that by linearity of the objective, the two observations cancell out each other at the end, to give the same effect as a single observation of $(x^t,f^t)$. Moreover, the leader after observing $(x^t,\pred^t)$ is equal to $a^t$, whilst after observing $(x^t,f^t-\pred^t)$ is equal to $b^t$. Thus by applying Lemma \ref{lem:pbtl} to this augmented sequence we get:
\begin{align*}
\sum_{t=1}^T \left(\pred^t(a^t)+ f^t(b^t)-\pred^t(b^t)\right)
\leq~& \sum_{t=1}^T \left(\pred^t(\pi^*(x^t))+f^t(\pi^*(x^t))-\pred^t(\pi^*(x^t))\right) + \error\\
=~& \sum_{t=1}^T f^t(\pi^*(x^t)) + \error
\end{align*}
Let $\BTPL_{\pred}^t =  \pred^t(a^t)+ f^t(b^t)-\pred^t(b^t)$ and $\FTPL^t = f^t(a^t)$. Then, observe that:
\begin{equation}
\FTPL^t - \BTPL_Q^t = f^t(a^t)-\pred^t(a^t)-(f^t(b^t)-\pred^t(b^t)) = \Delta Q^t(a^t) - \Delta Q^t(b^t)
\end{equation} 
Combining the two properties we get that for any policy $\pi^*$:
\begin{align*}
\sum_{t=1}^T \FTPL^t \leq~&  \sum_{t=1}^T\left(\Delta Q^t(a^t) - \Delta Q^t(b^t)\right)+ \sum_{t=1}^T \BTPL_Q^t \\
\leq~&  \sum_{t=1}^T\left(\Delta Q^t(a^t) - \Delta Q^t(b^t)\right) + \sum_{t=1}^T f^t(\pi^*(x^t)) + \error
\end{align*}
Re-arranging and taking expectation concludes the proof.
\end{proof}

\begin{lemma}[Stability with Predictors] In the transductive setting:
\begin{equation}
 \E\left[\Delta Q^t(a^t)-\Delta Q^t(b^t)\right]\leq 4\epsilon K \|f^t-\pred^t\|_{*}^2
\end{equation}
In the small separator setting:
\begin{equation}
 \E\left[\Delta Q^t(a^t)-\Delta Q^t(b^t)\right]\leq 4\epsilon Kd \|f^t-\pred^t\|_{*}^2
\end{equation}
\end{lemma}
\begin{proof}
We prove the first part of the Lemma. The second follows along identical arguments. 
By the definition of $\|f^t-\pred^t\|_{*} = \|\Delta Q^t\|_{*} = \max_{a\in \A} |\Delta Q^t(a)|$, we have:
\begin{align*}
\E_{\{z\}}\left[\Delta Q^t(a^t)-\Delta Q^t(b^t)\right] \le~& 2\|\Delta Q^t\|_{*} \Pr\left[a^t\neq b^t\right]
\end{align*}
Now observe that:
\begin{align*}
\Pr[a^t\neq b^t] \leq~& \sum_{j\in K} \left(\Pr[j\in a^t, j\notin b^t]+\Pr[j\notin a^t, j\in b^t]\right)
\end{align*}

We bound  the probability $\Pr[j\in a^t, j\notin b^t]$. We condition on all random variables of $\{z\}$ except for the random variable $\ell_{x^t}(j)$, i.e. the random loss placed at coordinate $j$ on the sample associated with context $x^t$. Denote the event corresponding to an assignment of all these other random variables as ${\cal E}_{-x^tj}$. Let $\ell_{x^tj}$ denote a loss vector which is $\ell_{x^t}(j)$ on the $j$-th coordinate and zero otherwise. Also let:
\begin{equation}
\Phi(\pi) = \sum_{\tau=1}^{t-1} f^\tau(\pi(x^\tau)) +Q^t(\pi(x^t))+ \sum_{x\in X-\{x^t\}} \dotp{\pi(x)}{\ell_x}+\dotp{\pi(x^t)}{\ell_{x^t}-\ell_{x^tj}}
\end{equation}
Let $\pi^* = \argmin_{\pi\in \Pi:j\in \pi(x^t)} \Phi(\pi)$ and $\tilde{\pi} = \min_{\pi\in \Pi: j\notin \pi(x^t)}  \Phi(\pi)$. The event that $\{j\in a^t\}$ happens only if:
\begin{equation}
\Phi(\pi^*)+\ell_{x^t}(j) \leq \Phi(\tilde{\pi})
\end{equation}
Let  and $\nu = \Phi(\tilde{\pi})-\Phi(\pi^*)$. Thus $j\in a^t$ only if:
\begin{equation}
\ell_{x^t}(j) \leq \nu
\end{equation}
Now if:
\begin{equation}
\ell_{x^t}(j) < \nu-2\|\Delta Q^t\|_{*}
\end{equation}
then it is easy to see that $\{j\in b^t\}$, since an extra loss of $f^t(a)-Q^t(a)\leq \|\Delta Q^t\|_{*}$ cannot push $j$ out of the optimal solution. More elaborately, for any other policy $\pi\in \Pi$, such that $j\notin \pi(x^t)$, the loss of $\pi^*$ including time-step $t$ is bounded as:
\begin{align*}
\Phi(\pi^*)+ \ell_{x^t}(j)+f^t(\pi^*(x^t))-Q^t(\pi^*(x^t))<~& \Phi(\pi)-2\|\Delta Q^t\|_{*}+f^t(\pi^*(x^t))-Q^t(\pi^*(x^t))\\
<~& \Phi(\pi)-\|\Delta Q\|_{*}\\
<~& \Phi(\pi) +f^t(\pi(x^t))-Q^t(\pi(x^t))
\end{align*}
Thus any policy $\pi$, such that $j\notin \pi(x^t)$ is suboptimal after seeing the loss at time-step $t$. Thus
\begin{align*}
\Pr[j\in a^t, j\notin b^t~|~{\cal E}_{-x^tj}]\leq \Pr[\ell_{x^t}(j) \in [\nu-2\|\Delta Q^t\|_{*},\nu]~|~{\cal E}_{-x^tj}]
\end{align*}
Since all other random variables are independent of $\ell_{x^t}(j)$ and $\ell_{x^t}(j)$ is a Laplacian with 
parameter $\epsilon$:
\begin{align*}
\Pr[\ell_{x^t}(j) \in [\nu-2\|\Delta Q^t\|_{*},\nu]~|~{\cal E}_{-x^tj}]=~&\Pr[\ell_{x^t}(j) \in [\nu-2\|\Delta Q^t\|_{*},\nu]]\\
 =~&\frac{\epsilon}{2}\int_{\nu-2\|\Delta Q^t\|_{*}}^{\nu}e^{-\epsilon|z|}dz\leq \frac{\epsilon}{2}\int_{\nu-2\|\Delta Q^t\|_{*}}^{\nu}dz \leq \epsilon\cdot \|\Delta Q^t\|_{*}
\end{align*}
Similarly it follows that that: $\Pr[j\notin \pi^t(x^t) \text{ and } j\in \pi^{t+1}(x^t)]\leq\epsilon\cdot \|\Delta Q^t\|_{*}$. 
To sum we get that:
\begin{align*}
\E_{\{z\}}\left[\Delta Q^t(a^t)-\Delta Q^t(b^t)\right]\le~& 2\|\Delta Q^t\|_{*}\Pr\left[\pi^t(x^t) \neq \pi^{t+1}(x^t)\right]\leq 4 \epsilon K \|\Delta Q^t\|_{*}^2
\end{align*}
\end{proof}

The expected error term is identical to the expected error that we upper bounded in Lemma \ref{lem:laplace-error-bound-lin}, hence the same bound carries over. Combining the above Lemmas with this observation, yields Theorem \ref{thm:oftpl-bound}.

\section{Infinite Policy Classes}
\label{sec:infinite}
In this section we focus on the contextual experts problem but consider infinite policy classes. 
Recall that in this setting, in each round $t$, the adversary picks a context $x^t \in {\cal X}$ and a loss function $\ell^t \in \R_{\geq 0}^K$, the learner, upon seeing the context $x^t$, chooses an action $a^t \in [K]$, and then suffers loss $\\ell^t(a^t)$.
We showed that as a simple consequence of Theorem~\ref{thm:main-complete}, that when competing with a set of policies $\Pi \subset ({\cal X}\rightarrow [K])$ with $|\Pi| = N$ and against an adaptive adversary, $\contlin$ has regret at most $O(d^{1/4}\sqrt{T\log(N)})$ in the transductive setting and regret at most $O(d^{3/4}\sqrt{KT\log(N)})$ in the non-transductive setting with small separator.

Here we consider the situation where the policy class $\Pi$ is infinite in size, but has small Natarajan dimension, which generalizes VC-dimension to multiclass problems. 
Specifically, we prove two results in this section: First we show that in the transductive case, $\contlin$ can achieve low regret relative to a policy class with bounded Natarajan dimension.
Then we show that in the non-transductive case, it is hard in an information-theoretic sense to achieve sublinear regret relative to a policy class with constant Natarajan dimension. 
Together, these results show that finite Natarajan or VC dimension is sufficient for sublinear regret in the transductive setting, but it is \emph{insufficient} for sublinear regret in the fully online setting. 

Before proceeding with the two results, we must introduce the notion of Natarajan dimension, which requires some notation.
For a class of functions $\mathcal{F}$ from $\mathcal{X} \rightarrow [K]$ and for a sequence $X = (x_1, \ldots, x_n) \in \mathcal{X}^n$, define $\mathcal{F}_X = \{ (f(x_1), \ldots, f(x_n)) \in [K]^n : f \in \mathcal{F} \}$ be the restriction of the functions to the points. 
Let $\Psi$ be a family of mappings from $[K] \rightarrow \{0, 1, \star\}$.
Let $\bar{\psi} = (\psi_1, \ldots, \psi_n) \in \Psi^n$ be a fixed sequence of such mappings and for a sequence $(s_1, \ldots, s_n) \in [K]^N$ define $\bar{\psi}(s) = (\psi_1(s_1), \ldots, \psi_1(s_n)) \in \{0,1,\star\}^n$. 
We say a sequence $X \in \mathcal{X}^n$ is $\Psi$-shattered by $\mathcal{F}$ if there exists $\bar{\psi} \in \Psi^n$ such that:
\begin{align*}
\{0,1\}^n \subseteq \{\bar{\psi}(s) : s \in \mathcal{F}_X\}
\end{align*}
The $\Psi$-dimension of a function class $\mathcal{F}$ is the largest $n$ such that there exist a sequence $X \in \mathcal{X}^n$ that is $\Psi$-shattered by $\mathcal{F}$. 
Notice that if $K=2$ and $\Psi$ contains only the identity map, then the $\Psi$-dimension is exactly the VC dimension. 

The \textbf{Natarajan dimension} is the $\Psi$ dimension for the class $\Psi_N = \{\psi_{N,i,j}, i,j \in [K], j \ne i\}$ where $\psi_{N,i,j}(a) = 1$ if $a=i$, $\psi_{N,i,j}(a) = 0$ if $a=j$ and $\psi_{N,i,j}(a) = \star$ otherwise.
Notice that Natarajan dimension is a strict generalization of VC-dimension as $\Psi_N$ contains only the identity map if $K=2$. 
Thus our result also applies to VC-classes in the two-action case.
The main property we will use about function classes with bounded Natarajan Dimension is the following analog of the Sauer-Shelah Lemma:
\begin{lemma}[Sauer-Shelah for Natarajan Dimension~\cite{haussler1995generalization,bendavid1995characterizations}]
\label{lem:natarajan_bound}
Suppose that $\mathcal{F}$ has $\Psi_N$ dimension at most $\nu$. 
Then for any set $X \in \mathcal{X}^n$, we have:
\begin{align*}
|\mathcal{F}_X| \le \left(\frac{ne(K+1)^2}{2\nu}\right)^\nu
\end{align*}
\end{lemma}

Our positive result for transductive learning with a Natarajan class is the following regret bound for $\contlin$,
\begin{corollary}
\label{cor:natarajan}
Consider running $\contlin(X,\epsilon)$ in the transductive contextual experts setting with a policy class $\Pi$ with Natarajan dimension at most $\nu$. Then the algorithm achieves regret against an adaptive and adversarially chosen sequence of contexts and loss functions,
\begin{align*}
\epsilon\sum_{t=1}^T\E[\dotp{\pi^t(x^t)}{\ell^t}^2] + \frac{10}{\epsilon}\sqrt{d\nu\log(K)\log\left(\frac{de(K+1)^2}{2\nu}\right)}.
\end{align*}
When $\epsilon$ is set optimally and losses are in $[0,1]^K$, this is $O((d\nu\log(K)\log(dK/\nu))^{1/4}\sqrt{T})$.
\end{corollary}
\begin{proof}
The result is a consequence of the second clause of Theorem~\ref{thm:main-complete}, using the additional fact that any sequence of contexts $X = (x_1,\ldots,x_d)$ induce a finite policy class $\Pi_X \subseteq [K]^d$.
The fact that $\Pi$ has Natarajan dimension at most $\nu$ means that $|\Pi_X| \le \left(\frac{de(K+1)^2}{2\nu}\right)^\nu$ by Lemma~\ref{lem:natarajan_bound}.
Therefore, once the $d$ contexts are fixed, as they are in the transductive setting, we are back in the finite policy case and can apply Theorem~\ref{thm:main-complete} with $N$ replaced by $|\Pi_X|$. 
\end{proof}

Thus we see that $\contlin$ has sublinear regret relative to policy classes with bounded Natarajan dimension, even against adaptive adversaries.
The second result in this section shows that this result cannot be lifted to the non-transductive setting. 
Specifically, we prove the following theorem in the section, which shows that no algorithm, including inefficient ones, can achieve sublinear regret against a VC class in the non-transductive setting.
\begin{theorem}\label{thm:hardness}
Consider an online binary classification problem in one dimension with $\mathcal{F} \subset [0,1] \rightarrow \{0,1\}$ denoting the set of all threshold functions.
Then there is no learning algorithm that can guarantee $o(T)$ expected regret against an adaptive adversary.
In particular, there exists a policy class of VC dimension one such that no learning algorithm can achieve sublinear regret against an adaptive adversary in the contextual experts problem. 
\end{theorem}
\begin{proof}
We define an adaptive adversary and argue that it ensures at least $1/2$ expected regret per round. 
While the adversary does not have access to the random coins of the learner, it can compute the probability that the learner would label any point as $\{0,1\}$. 
At round $t$, let $p_t(x)$ denote the probability that the learner would label a point $x \in [0,1]$ as $1$, and note that this quantity is conditioned on the entire history of interaction. 
At each round $t$, the adversary will have played a set of points $X_t^+$ with positive label and $X_t^-$ with negative label and she will maintain the invariant that $\min_{x \in X_t^+} x > \max_{x \in X_t^-} x$ for all $t$.
At every time $t$, the adversary will play context $x_t \in (\max_{x \in X_t^-}x, \min_{x \in X_t^+} x)$.
The adversary, knowing the learning algorithm, will compute $p_t(x_t)$ and assign label $y_t = 1$ if $p_t(x) < 1/2$ and $0$ otherwise. 
The adversary will then update the sets $X_{t+1}^+ \gets X_t^+ \cup \{x_t\}$ if $y_t = 1$ and $X_{t+1}^+ \gets X_t^+$ otherwise.
$X_{t+1}^-$ is updated analogously.

Clearly this sequence of contexts maintains the appropriate invariant for the adversary, namely there is always an interval between the positive and negative examples in which he can pick a context.
This implies that on the sequence, there is a threshold $f^\star \in \mathcal{F}$ that perfectly classifies the points, so its cumulative reward is $T$. 
Moreover, by the choice of label selected by the adversary, the expected reward of the learner at round $t$ is at most $1/2$, which means the cumulative expected reward of the learner is at most $T/2$. 
Thus the regret of the learner is at least $T/2$. 
\end{proof}


%









%

\end{appendix}

\end{document}